\documentclass{article}

\usepackage[preprint]{neurips_2026}

\usepackage[utf8]{inputenc} 
\usepackage[T1]{fontenc}    
\usepackage{hyperref}       
\usepackage{url}            
\usepackage{booktabs}       
\usepackage{amsfonts, amsmath, amssymb, amsthm}       
\usepackage{nicefrac}       
\usepackage{microtype}      
\usepackage{xcolor}         
\usepackage{algorithm}
\usepackage{algpseudocode}
\usepackage{float}
\usepackage{graphicx}
\usepackage[table]{xcolor}
\usepackage{subcaption}
\usepackage{cancel}
\usepackage{enumitem}

\usepackage{thmtools, thm-restate}

\newtheorem{proposition}{Proposition}
\newtheorem{lemma}{Lemma}
\newtheorem{remark}{Remark}

\newcommand{\diff}{\mathrm{d}}
\DeclareMathOperator*{\argmin}{arg\,min}

\title{Noise, Denoise, Correct: MCMC Posterior Sampling with Diffusion Priors in Three Steps}
\workshoptitle{AI for Stochastic Dynamics}

\author{%
  So Takao \qquad Gregory David Bellchambers \qquad Luke Ye \\
  \bf Sanmitra Ghosh$^{*}$ \qquad Michalis Michaelides$^{*}$ \\
  PhysicsX \\London, UK \\
  \texttt{so.takao@physicsx.ai} \\
  {\small $^{*}$Joint senior author}
}

\begin{document}

\maketitle

\begin{abstract}
  Pretrained diffusion models are powerful priors for inverse problems, but posterior sampling under nonlinear, non-differentiable forward models remain hard. We introduce diffusion waltz, an MCMC method using SDEdit-style noising–denoising as a proposal, corrected via Metropolis–Hastings for exact posterior sampling without prior evaluation. We further propose injecting observations into the proposal while preserving exactness, using a gradient-free ensemble Kalman update. On a non-differentiable Navier–Stokes initial condition recovery task, diffusion waltz outperforms existing baselines across different noise and nonlinearity regimes.
\end{abstract}

\section{Introduction}
Solving inverse problems is a central challenge across science and engineering, where the goal is to recover an unknown state $x_0$ from noisy, incomplete observations $y = \mathcal{A}(x_0) + \epsilon$; here, $\mathcal{A}$ is a known forward operator and $\epsilon \sim \mathcal{N}(0, \sigma_y^2I)$ for some $\sigma_y > 0$. Such observations typically do not carry enough information to recover $x_0$ exactly, rendering the problem ill-posed. The Bayesian formulation resolves this by placing a prior $p_0(x_0)$ over the state and targeting the posterior \citep{stuart2010inverse}:
\begin{align}\label{eq:posterior}
    p(x_0 \mid y) \propto p(y \mid x_0) \, p_0(x_0).
\end{align}
A growing body of work seeks to use pretrained generative models as priors \citep{daras2024survey}, but posterior sampling remains elusive when the prior is accessible only through simulation, with no tractable density. Methods such as classifier and classifier-free guidance \citep{dhariwal2021diffusion, ho2022classifier} can sample from (\ref{eq:posterior}) approximately. However, they require additional training and cannot adapt to new forward models. Inference-time guidance methods \citep{chung2022diffusion, song2023pseudoinverse, rozet2024learning, kim2025flowdps, zhang2025improving} adapt to new problems without retraining, but rely on approximations to the likelihood score, leading to biased posterior approximations; unlike in computer vision, where many of these methods were originally devised, scientific applications often demand accurate posterior estimation. Furthermore, many of these methods rely on taking gradients through the likelihood, requiring $\mathcal{A}$ to be differentiable -- an assumption that fails in many realistic applications.

To address this, we propose {\em diffusion waltz}, a three-step cycle to sample from (\ref{eq:posterior}) via Markov Chain Monte Carlo (MCMC), which targets the posterior exactly, assuming a perfect denoiser. Our proposal kernel $q(x_0' \mid x_0)$, inspired by SDEdit \citep{mengsdedit}, is constructed by (i) noising $x_0$ to an appropriate level $t$, then (ii) denoising back to level $0$ to yield a new sample $x_0'$. This is followed by (iii) a Metropolis-Hastings (MH) accept/reject step, where the acceptance probability crucially does not depend on the prior density $p_0(x_0)$.
To accelerate convergence, we further propose a way to inject information about $y$ into the proposal without requiring differentiability of $\mathcal{A}$, such that the modified proposal still yields an exact MH correction that does not involve the prior density.

\section{Method}

\subsection{Diffusion waltz for posterior sampling}
We first describe the basic variant of diffusion waltz, which we refer to as {\em random walk Metropolis-Hastings (RWMH) waltz}.
Let $p_0, p_1$ be two probability distributions, representing the data and reference distributions, respectively. We consider a stochastic interpolant $x_t = \alpha_t x_0 + \beta_t x_1$ for $t \in [0, 1]$, where $x_0 \sim p_0$, $x_1 \sim p_1$, and $\alpha_t, \beta_t$ are $C^2$-smooth, monotonic functions satisfying $\alpha_0 = \beta_1 = 1$ and $\alpha_1 = \beta_0 = 0$, e.g., $\alpha_t = 1 - t$ and $\beta_t = t$. 
A generative SDE/ODE can then be trained to sample from $p_\theta(x_0) \approx p_0(x_0)$ via score/flow matching \citep{song2020score,lipman2022flow,albergo2025stochastic} (see Appendix~\ref{app:stochastic-interpolants}).

For a fixed $t \in (0, 1)$ and an initial sample $x_0 \sim p_\theta$, RWMH waltz cycles through three steps:
\begin{enumerate}
    \item Add noise to $x_0$ via the interpolant $x_t = \alpha_t x_0 + \beta_t x_1$, where $x_1 \sim p_1$.
    \item Denoise $x_t$ by proposing $x_0' \sim p_\theta(x_0 \mid x_t) \propto p(x_t \mid x_0) \, p_\theta(x_0)$.
    \item Accept $x_0'$ with probability $\alpha = \min\Big\{1, \,\frac{p(y \mid x_0')}{p(y \mid x_0)}\Big\}$ and set $x_0 \leftarrow x_0'$ to repeat the cycle.
\end{enumerate}
The first two steps -- perturbing $x_0$ via noising/denoising to produce a new sample $x_0'$ that remains close to it on the data manifold -- is common in image editing, as exemplified by SDEdit \citep{mengsdedit}. Here we repurpose this idea to define a proposal kernel $q(x_0' \mid x_0)$ for MCMC, allowing random walk-like exploration of the data manifold. The third step is a Metropolis–Hastings (MH) correction, which guarantees that the resulting chain converges to the posterior as the number of iterations grows. Crucially, the acceptance probability in this step depends {\em only} on the likelihood; thus, we never have to evaluate the prior density $p_\theta$. This follows from the reversibility of the proposal kernel with respect to the generative prior, i.e., $q(x_0' \mid x_0)p_\theta(x_0) = q(x_0 \mid x_0')p_\theta(x_0')$, which cancels the prior terms in the acceptance ratio (see Appendix \ref{app:rwmh-acceptance}).

In principle, RWMH waltz samples exactly from the posterior. In practice, however, the denoising step requires approximation; e.g., by 
solving a denoising SDE via Euler-Maruyama (see Appendix~\ref{eq:sde-choice}), or by directly targeting the denoising kernel $p(x_0 \mid x_t)$ using recent advances in stochastic flow maps \citep{potaptchik2026meta, holderrieth2026diamond, mccallum2026strong, pan2026maps}. The SDE discretization error in the former approach can, in principle, be corrected by combining predictor–corrector-style sampling with further MH corrections \citep{lam2026metropolis}; in practice, though, this is usually unnecessary to obtain satisfactory results.

\subsection{Conditional noising}

We now propose an extension that takes into account observation information into the proposal. Since the RWMH waltz proposal explores the data manifold randomly around the current point $x_0$ in the data manifold, it typically suffers from low acceptance rates and slow mixing, much like standard RWMH. Typically, this issue is alleviated by introducing observation information directly into the proposal; e.g., MALA \citep{roberts1996exponential} and HMC \citep{duane1987hybrid}. We seek an analogous approach within the diffusion waltz framework.

To this end, we consider adding information about $y$ into the {\em noising step} of RWMH waltz. In its most general form, the noising step is replaced by $x_t \sim q(x_t \mid x_0, y)$ for an arbitrary noising kernel $q$ that depends on $y$. Remarkably, the resulting MH acceptance probability still does not require evaluating the prior density $p_\theta$, making the method fully tractable provided $q$ can be evaluated (see line 3 in Algorithm~\ref{alg:conditional-noising-waltz} and Appendix~\ref{app:cn-waltz-acceptance} for the derivation of this acceptance probability).

\subsection{The ensemble Kalman conditional noising (EKCN) kernel}

As an example of such a conditional noising kernel $q(x_t \mid x_0, y)$, consider
\begin{align}
x_t &= \alpha_t \big(x_0 + g(x_0, y; h)\big)+ \beta_t x_1, \quad x_1 \sim p_1,
\end{align}
where $g(x_0, y; h)$ nudges $\mathcal{A}(x_0)$ towards $y$ with learning rate $h > 0$. In the {\em ensemble Kalman conditional noising (EKCN) kernel}, we take $g$ to be an ensemble Kalman Levenberg-Marquardt update rule \citep[Algorithm 13.15]{sanz2023inverse}
\begin{align}
    g(x_0, y; h) = K(h) \big(y - \mathcal{A}(x_0)\big),
\end{align}
where the {\em Kalman gain} $K(h)$ can be estimated from an ensemble of particles without requiring gradients of $\mathcal{A}$ (see Appendix \ref{app:ensemble-kalman}). In practice, we run an ensemble of $J$ MCMC chains in parallel to compute $K(h)$ and re-use it to compute the MH acceptance probability (see Appendix~\ref{app:ekcn-kernel}). Adaptive tuning of $t, h$ via Robbins-Monro is also applied, allowing smaller steps as chains converge (see Appendix \ref{app:adaptive-tuning}).
EKCN waltz can thus be understood as cycling over (i) a Kalman update to minimise data misfit, (ii) SDEdit to project onto the data manifold, and (iii) MH correction.

\begin{figure}[t]
\centering
\begin{minipage}[t]{0.46\textwidth}
\begin{algorithm}[H]
\caption{\small Diffusion waltz (random walk)}
\begin{algorithmic}[1]
\small
    \State Add noise $x_t \sim p(x_t \mid x_0)$ via interpolant
    \State Denoise $x_0' \sim p_\theta(x_0 \mid x_t)$
    \State Accept $x_0'$ with probability
    \Statex \qquad\qquad $\alpha = \min\left\{1, \dfrac{p(y \mid x_0')}{p(y \mid x_0)}\right\}$
    \Statex (Set $x_0 \leftarrow x_0'$ if accepted, else keep $x_0$)
\end{algorithmic}
\label{alg:random-walk-waltz}
\label{alg:rwmh-waltz}
\end{algorithm}
\end{minipage}
\hfill
\begin{minipage}[t]{0.52\textwidth}
\begin{algorithm}[H]
\caption{\small Diffusion waltz (conditional noising)}
\begin{algorithmic}[1]
\small
    \State Propose $x_t \sim q(x_t \mid x_0, y)$ for a chosen kernel $q$
    \State Denoise $x_0' \sim p_\theta(x_0 \mid x_t)$
    \State Accept $(x_t, x_0')$ with probability
    \Statex $\alpha = \min\left\{1, \dfrac{p(y \mid x_0')\,q(x_t \mid x_0', y)\,p(x_t \mid x_0)}{p(y \mid x_0)\,q(x_t \mid x_0, y)\,p(x_t \mid x_0')}\right\}$
    \Statex (Set $x_0 \leftarrow x_0'$ if accepted, else keep $x_0$)
\end{algorithmic}
\label{alg:conditional-noising-waltz}
\end{algorithm}
\end{minipage}
\end{figure}

\section{Related works}
The RWMH waltz proposal (Algorithm~\ref{alg:random-walk-waltz}) has been considered in concurrent work \citep{kang2026sampling, phunyaphibarn2026reward}, but mechanisms to inject likelihood information into the proposal have not yet been explored.
The idea of avoiding prior density computation by designing reversible proposals underlies preconditioned Crank–Nicolson (pCN) and pCN-Langevin \citep{cotter2013mcmc}.
In addition to conditional noising, another way to inject information about $y$ into the waltz proposal is to add it during the denoising step; similar ideas have been explored by DAPS \citep{zhang2025improving} and RePS \citep{ahmed2025solving}, which cycles through noising and conditional de-noising, although these require gradients of $\mathcal{A}$.
Other MCMC-based posterior samplers with pretrained diffusion priors include sequential Monte Carlo \citep{wu2023practical, cardoso2024monte, singhal2025general}, split-Gibbs \citep{coeurdoux2024plug, wu2024principled, zheng2025blade}, and elliptical slice sampling \citep{kalaivanan2025ess}, several of which we compare against in the next section.

\section{Experiments}\label{sec:experiments}

We consider the Navier-Stokes initial condition recovery problem to validate the performance of diffusion waltz against various baselines.
An initial vorticity field $\omega_0$ is evolved forward in time by a Navier-Stokes pseudo-spectral solver to obtain $\omega_T$. This is then downsampled and corrupted by Gaussian noise to generate noisy observations $y$. The goal of the experiment is to reconstruct the initial vorticity $\omega_0$ given $y$, i.e., infer $p_\theta(\omega_0 \mid y) \propto p(y \mid \omega_0)p_\theta(\omega_0)$, where the diffusion prior $p_\theta(\omega_0)$ is  trained on long time trajectories of $\omega_t$. 
We explore different regimes in terms of the level of nonlinearity (controlled by $T$) and noise intensity (controlled by $\sigma_y$).
Details in Appendix \ref{app:experiment-details}.

\paragraph{Baseline comparisons.} Since we use a non-differentiable numerical solver to simulate the Navier-Stokes system, we restrict ourselves to derivative-free methods. We compare both RWMH waltz (DW-RWMH) and EKCN waltz (DW-EKCN) against a mix of standard and state-of-the-art derivative-free methods, including EKI \citep{kovachki2019ensemble}, EnKG \citep{zheng2024ensemble}, FK Steering \citep{singhal2025general}, ESS-Flow \citep{kalaivanan2025ess}, Blade \citep{zheng2025blade} and MFM-GF \citep{potaptchik2026meta}. Details of the baselines are provided in Appendix \ref{app:baselines}.

We display results in Table \ref{tab:ns-results}, where we compare the relative $\ell_2$ score (Rel $\ell_2$), continuous ranked probability score (CRPS), and spread-skill ratio (SSR), providing a mix of reconstruction and probabilistic metrics (see Appendix~\ref{app:metricx}). Visualizations in the high nonlinearity / high noise regime are shown in Figure \ref{fig:results-t5-s1.0} (for other regimes, see Figures \ref{fig:results-t1-s0.1}--\ref{fig:results-t5-s0.1} in the appendix). Overall, DW-EKCN performs strongly on all scores across all regimes, consistently ranking first or second. We note the significant improvement over EKI, highlighting the importance of the SDEdit+MH step in DW for keeping ensemble members on-manifold. EnKG, another ensemble Kalman-based method, performs reasonably in low-noise regimes but suffers from low spread as seen by the low SSR; in the high noise regime, it performs poorly. Blade, an improvement over EnKG that combines EKI-like statistical linearisation with split-Gibbs sampling, performs competitively, especially at low nonlinearity; however, it suffers from slow inference speed due to the need for multiple forward model passes during Langevin sampling. ESS-Flow, despite strong performance, also suffers from slow inference speed due to the large number of forward passes required during shrinking. Finally, FK Steering and MFM-GF, which rely on importance sampling, both exhibit weight collapse, leading to poor performance even in the low nonlinearity / low noise regime.

For comparisons against standard gradient-based guidance methods (e.g. DPS) using a differentiable neural surrogate for the NS system, see Appendix~\ref{app:gradient-based-approaches}; we see that DW-EKCN and other MCMC-based approaches (ESS-Flow, Blade) mostly outperform these methods.

\vspace{-0.3em}

\paragraph{Effect of conditional noising.} Comparing the two DW variants, we see that conditional noising is instrumental in accelerating convergence to the stationary distribution. We display trace plots of the potential $U(x_0) = -\log p(y \mid x_0)$ in Figure \ref{fig:ns-mcmc-trace}, showing that the slow convergence of RWMH waltz is due to many chains getting stuck early on, especially in the highly nonlinear setting ($T=5$). EKCN guides each ensemble member towards high-likelihood regions, leading to rapid convergence within the first $\sim 50$ MCMC steps. This also leads to high acceptance rates early in the iteration (Figure \ref{fig:ns-acc-rate}), although the rate decays rapidly once the chains converge; e.g., in the high-noise setting, it settles to between $10\sim20\%$, and in the low-noise regime, less than $5\%$. Investigating methods to maintain high acceptance rates is a promising direction for future work.

\begin{table}[t]
\centering
\caption{\small Mean Rel-$\ell_2$, CRPS and SSR across five unseen test examples (for an expanded version with standard deviations and gradient-based baselines, see Tables \ref{tab:t1-results}--\ref{tab:t5-results}). Best results in {\bf{bold}}, second best \underline{underlined}.}
\resizebox{\textwidth}{!}{%
\begin{tabular}{l ccc ccc ccc ccc}
\toprule
 & \multicolumn{6}{c}{Low nonlinearity ($T=1$)} & \multicolumn{6}{c}{High nonlinearity ($T=5$)} \\
\cmidrule(lr){2-7} \cmidrule(lr){8-13}
 & \multicolumn{3}{c}{Low noise} & \multicolumn{3}{c}{High noise} & \multicolumn{3}{c}{Low noise} & \multicolumn{3}{c}{High noise} \\
\cmidrule(lr){2-4} \cmidrule(lr){5-7} \cmidrule(lr){8-10} \cmidrule(lr){11-13}
\textbf{Methods} & Rel $\ell_2\downarrow$ & CRPS$\downarrow$ & SSR$\rightarrow 1$ & Rel $\ell_2\downarrow$ & CRPS$\downarrow$ & SSR$\rightarrow 1$ & Rel $\ell_2\downarrow$ & CRPS$\downarrow$ & SSR$\rightarrow 1$ & Rel $\ell_2\downarrow$ & CRPS$\downarrow$ & SSR$\rightarrow 1$ \\
\midrule
EKI                & 0.517 & 0.370 & 0.024 & 0.832 & 0.560 & 0.126 & 1.059 & 0.754 & 0.031 & 0.865 & 0.522 & 0.267 \\
EnKG               & 0.129 & 0.086 & 0.076 & 0.617 & 0.414 & 0.071 & \underline{0.418} & 0.279 & 0.098 & 0.812 & 0.523 & 0.084 \\
FK Steering    & 0.787 & 0.538 & 0.012 & 0.782 & 0.500 & 0.101 & 1.383 & 1.000 & 0.016 & 0.992 & 0.511 & 1.084  \\
ESS-Flow           & 0.178 & 0.080 & 1.405 & \underline{0.248} & \underline{0.118} & 1.179 & 0.506 & 0.242 & \bf{0.987} & 0.491 & 0.233 &\bf{0.987} \\
Blade              & \bf{0.069} & \bf{0.037} & 1.797 & 0.296 & 0.139 & \bf{0.982} & 0.421 & \underline{0.239} & 2.312 & \underline{0.442} & \underline{0.231} & 1.326 \\
MFM-GF             & 0.406 & 0.191 & \bf{1.395} & 0.425 & 0.202 & 1.406 & 0.549 & 0.262 & 1.245 & 0.547 & 0.261 & 1.253 \\
\rowcolor{gray!15}
DW-RWMH            & 0.451 & 0.177 & 1.703 & 0.550 & 0.229 & 1.476 & 0.810 & 0.398 & 1.163 & 0.815 & 0.396 & 1.155 \\
\rowcolor{gray!15}
DW-EKCN       & \underline{0.086} & \underline{0.045} & \underline{1.396} & \bf{0.220} & \bf{0.101} & \underline{1.081} & \bf{0.247} & \bf{0.152} & \underline{1.054} & \bf{0.435} & \bf{0.207} & \underline{0.945} \\
\bottomrule
\end{tabular}%
\label{tab:ns-results}
}
\end{table}

\begin{figure}[t]
    \centering
    \includegraphics[width=1.0\textwidth]{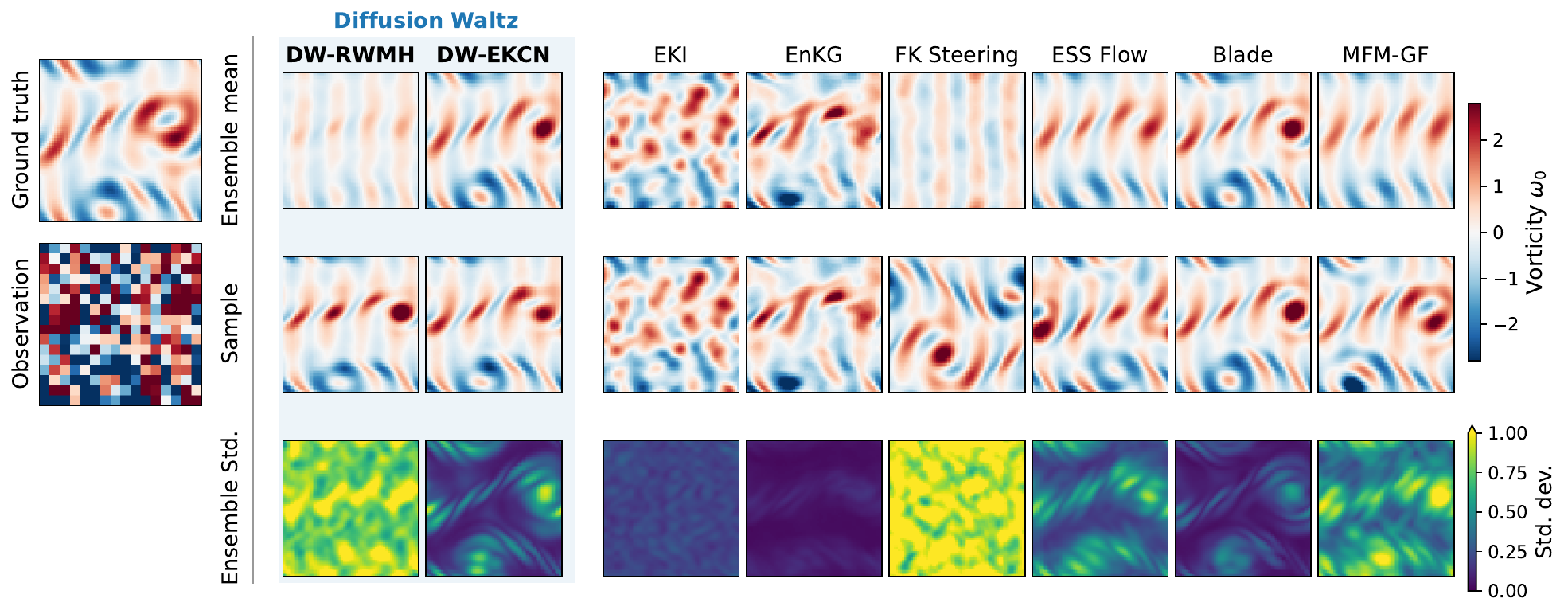}
    \caption{\small Results on the Navier-Stokes (NS) initial condition recovery problem in the high nonlinearity / high noise regime. Despite the heavily corrupted observation, diffusion waltz with ensemble Kalman conditional noising (DW-EKCN) recovers the ground truth accurately with well-calibrated uncertainties.}
    \label{fig:results-t5-s1.0}
\end{figure}

\vspace{-0.3em}
\section{Conclusion}
Diffusion waltz provides a simple plug-and-play framework for posterior sampling via MCMC that is asymptotically exact (assuming a perfect denoiser) and handles non-differentiable forward models. Its conditional noising extension injects observation information into the proposal while preserving this exactness, substantially improving mixing. On Navier–Stokes initial condition recovery, diffusion waltz matches or outperforms strong derivative-free baselines in both accuracy and calibration.

\begin{ack}
We would like to thank Martin Accou, Andrew Boardman, Thomas Gessey-Jones, Ayush Jain, James-Michael Leahy and Daniel Owen-Lloyd for the insightful discussions and careful review of the manuscript.

\href{https://credit.niso.org}{CRediT} author statement: \textbf{So Takao:} Conceptualization, Methodology, Software, Investigation,
Writing - Original Draft, Visualization. \textbf{Gregory David Bellchambers:} Software, Investigation. \textbf{Luke Ye:} Software, Investigation. \textbf{Sanmitra Ghosh}: Conceptualization, Resources, Writing - Original Draft,
Supervision, Project administration. \textbf{Michalis Michaelides}: Conceptualization, Resources, Writing - Original Draft,
Supervision, Project administration.
\end{ack}

\bibliographystyle{plainnat}
\bibliography{biblio}

@article{kang2026sampling,
  title={Sampling Data with Chains of Forward-Backward Diffusion Steps},
  author={Kang, Hyunmo and Levi, Noam Itzhak and Wegner, Corinna Elena and Korchinski, Daniel J and Wyart, Matthieu},
  journal={arXiv preprint arXiv:2605.27006},
  year={2026}
}

@article{phunyaphibarn2026reward,
  title={Reward-Guided Discrete Diffusion via Clean-Sample {M}arkov Chain for Molecule and Biological Sequence Design},
  author={Phunyaphibarn, Prin and Sung, Minhyuk},
  journal={arXiv preprint arXiv:2602.09424},
  year={2026}
}

@article{albergo2025stochastic,
  title={Stochastic interpolants: {A} unifying framework for flows and diffusions},
  author={Albergo, Michael and Boffi, Nicholas M and Vanden-Eijnden, Eric},
  journal={Journal of Machine Learning Research},
  volume={26},
  number={209},
  pages={1--80},
  year={2025}
}

@article{cotter2013mcmc,
  title={{MCMC} methods for functions: modifying old algorithms to make them faster},
  author={Cotter, Simon L and Roberts, Gareth O and Stuart, Andrew M and White, David},
  journal={Statistical Science},
  pages={424--446},
  year={2013},
  publisher={JSTOR}
}

@inproceedings{mengsdedit,
  title={{SDE}dit: Guided Image Synthesis and Editing with Stochastic Differential Equations},
  author={Meng, Chenlin and He, Yutong and Song, Yang and Song, Jiaming and Wu, Jiajun and Zhu, Jun-Yan and Ermon, Stefano},
  booktitle={International Conference on Learning Representations},
  year={2022}
}

@article{potaptchik2026meta,
  title={Meta flow maps enable scalable reward alignment},
  author={Potaptchik, Peter and Saravanan, Adhi and Mammadov, Abbas and Prat, Alvaro and Albergo, Michael S and Teh, Yee Whye},
  journal={arXiv preprint arXiv:2601.14430},
  year={2026}
}

@article{holderrieth2026diamond,
  title={Diamond maps: {E}fficient reward alignment via stochastic flow maps},
  author={Holderrieth, Peter and Chen, Douglas and Eyring, Luca and Shah, Ishin and Anantharaman, Giri and He, Yutong and Akata, Zeynep and Jaakkola, Tommi and Boffi, Nicholas Matthew and Simchowitz, Max},
  journal={arXiv preprint arXiv:2602.05993},
  year={2026}
}

@article{lam2026metropolis,
  title={Metropolis-Adjusted Diffusion Models},
  author={Lam, Kevin H and Farghly, Tyler and Williams, Christopher and Yang, Jun and Teh, Yee Whye and Doucet, Arnaud},
  journal={arXiv preprint arXiv:2605.09654},
  year={2026}
}

@article{mccallum2026strong,
  title={Strong stochastic flow maps},
  author={McCallum, Sam and Blasingame, Zander W and Herschell, Timothy and Rindtorff, Niklas and Tong, Alexander and Foster, James},
  journal={arXiv preprint arXiv:2606.01086},
  year={2026}
}

@article{pan2026maps,
  title={It\^o maps for any-step {SDEs}},
  author={Pan, Zhengkai and Potaptchik, Peter and Yao, Wenxi and Albergo, Michael S and Pidstrigach, Jakiw},
  journal={arXiv preprint arXiv:2606.11156},
  year={2026}
}

@article{ahmed2025solving,
  title={Solving Diffusion Inverse Problems with Restart Posterior Sampling},
  author={Ahmed, Bilal and Makin, Joseph G},
  journal={arXiv preprint arXiv:2511.20705},
  year={2025}
}

@article{roberts1996exponential,
  title={Exponential convergence of {L}angevin distributions and their discrete approximations},
  author={Roberts, Gareth O and Tweedie, Richard L},
  journal={Bernoulli},
  year={1996}
}

@article{iglesias2013ensemble,
  title={Ensemble {K}alman methods for inverse problems},
  author={Iglesias, Marco A and Law, Kody JH and Stuart, Andrew M},
  journal={Inverse Problems},
  volume={29},
  number={4},
  pages={045001},
  year={2013},
  publisher={IOP Publishing}
}

@article{kovachki2019ensemble,
  title={Ensemble {K}alman inversion: a derivative-free technique for machine learning tasks},
  author={Kovachki, Nikola B and Stuart, Andrew M},
  journal={Inverse Problems},
  volume={35},
  number={9},
  pages={095005},
  year={2019},
  publisher={IOP Publishing}
}

@article{rozet2024learning,
  title={Learning diffusion priors from observations by expectation maximization},
  author={Rozet, Fran{\c{c}}ois and Andry, G{\'e}r{\^o}me and Lanusse, Fran{\c{c}}ois and Louppe, Gilles},
  journal={Advances in Neural Information Processing Systems},
  volume={37},
  pages={87647--87682},
  year={2024}
}

@inproceedings{zhang2025improving,
  title={Improving diffusion inverse problem solving with decoupled noise annealing},
  author={Zhang, Bingliang and Chu, Wenda and Berner, Julius and Meng, Chenlin and Anandkumar, Anima and Song, Yang},
  booktitle={2025 IEEE/CVF Conference on Computer Vision and Pattern Recognition (CVPR)},
  pages={20895--20905},
  year={2025},
  organization={IEEE}
}

@article{wu2023practical,
  title={Practical and asymptotically exact conditional sampling in diffusion models},
  author={Wu, Luhuan and Trippe, Brian and Naesseth, Christian and Blei, David and Cunningham, John P},
  journal={Advances in Neural Information Processing Systems},
  volume={36},
  pages={31372--31403},
  year={2023}
}

@article{singhal2025general,
  title={A general framework for inference-time scaling and steering of diffusion models},
  author={Singhal, Raghav and Horvitz, Zachary and Teehan, Ryan and Ren, Mengye and Yu, Zhou and McKeown, Kathleen and Ranganath, Rajesh},
  journal={arXiv preprint arXiv:2501.06848},
  year={2025}
}

@article{coeurdoux2024plug,
  title={Plug-and-play split {G}ibbs sampler: embedding deep generative priors in {B}ayesian inference},
  author={Coeurdoux, Florentin and Dobigeon, Nicolas and Chainais, Pierre},
  journal={IEEE Transactions on Image Processing},
  volume={33},
  pages={3496--3507},
  year={2024},
  publisher={IEEE}
}

@article{wu2024principled,
  title={Principled probabilistic imaging using diffusion models as plug-and-play priors},
  author={Wu, Zihui and Sun, Yu and Chen, Yifan and Zhang, Bingliang and Yue, Yisong and Bouman, Katherine L},
  journal={Advances in Neural Information Processing Systems},
  volume={37},
  pages={118389--118427},
  year={2024}
}

@article{zheng2025blade,
  title={Blade: A derivative-free {B}ayesian inversion method using diffusion priors},
  author={Zheng, Hongkai and Wang, Austin and Wu, Zihui and Huang, Zhengyu and Baptista, Ricardo and Yue, Yisong},
  journal={arXiv preprint arXiv:2510.10968},
  year={2025}
}

@article{kalaivanan2025ess,
  title={{ESS-Flow}: {T}raining-free guidance of flow-based models as inference in source space},
  author={Kalaivanan, Adhithyan and Zhao, Zheng and Sj{\"o}lund, Jens and Lindsten, Fredrik},
  journal={arXiv preprint arXiv:2510.05849},
  year={2025}
}

@inproceedings{chung2022diffusion,
  title={Diffusion posterior sampling for general noisy inverse problems},
  author={Chung, Hyungjin and Kim, Jeongsol and Mccann, Michael T and Klasky, Marc L and Ye, Jong Chul},
  booktitle={International Conference on Learning Representations},
  year={2023}
}

@article{stuart2010inverse,
  title={Inverse problems: a {B}ayesian perspective},
  author={Stuart, Andrew M},
  journal={Acta numerica},
  volume={19},
  pages={451--559},
  year={2010},
  publisher={Cambridge University Press}
}

@article{dhariwal2021diffusion,
  title={Diffusion models beat {GANs} on image synthesis},
  author={Dhariwal, Prafulla and Nichol, Alexander},
  journal={Advances in neural information processing systems},
  volume={34},
  pages={8780--8794},
  year={2021}
}

@article{ho2022classifier,
  title={Classifier-free diffusion guidance},
  author={Ho, Jonathan and Salimans, Tim},
  journal={arXiv preprint arXiv:2207.12598},
  year={2022}
}

@article{daras2024survey,
  title={A survey on diffusion models for inverse problems},
  author={Daras, Giannis and Chung, Hyungjin and Lai, Chieh-Hsin and Mitsufuji, Yuki and Ye, Jong Chul and Milanfar, Peyman and Dimakis, Alexandros G and Delbracio, Mauricio},
  journal={arXiv preprint arXiv:2410.00083},
  year={2024}
}

@inproceedings{song2023pseudoinverse,
  title={Pseudoinverse-guided diffusion models for inverse problems},
  author={Song, Jiaming and Vahdat, Arash and Mardani, Morteza and Kautz, Jan},
  booktitle={International conference on learning representations},
  year={2023}
}

@inproceedings{cardoso2024monte,
  title={Monte Carlo guided Denoising Diffusion models for {B}ayesian linear inverse problems.},
  author={Cardoso, Gabriel and Le Corff, Sylvain and Moulines, Eric and others},
  booktitle={International Conference on Learning Representations},
  volume={2024},
  pages={44001--44037},
  year={2024}
}

@inproceedings{kim2025flowdps,
  title={Flow{DPS}: Flow-driven posterior sampling for inverse problems},
  author={Kim, Jeongsol and Kim, Bryan Sangwoo and Ye, Jong Chul},
  booktitle={2025 IEEE/CVF International Conference on Computer Vision (ICCV)},
  pages={12328--12337},
  year={2025},
  organization={IEEE}
}

@article{li2022learning,
  title={Learning chaotic dynamics in dissipative systems},
  author={Li, Zongyi and Liu-Schiaffini, Miguel and Kovachki, Nikola and Azizzadenesheli, Kamyar and Liu, Burigede and Bhattacharya, Kaushik and Stuart, Andrew and Anandkumar, Anima},
  journal={Advances in Neural Information Processing Systems},
  volume={35},
  pages={16768--16781},
  year={2022}
}

@article{duane1987hybrid,
  title={Hybrid {M}onte {C}arlo},
  author={Duane, Simon and Kennedy, Anthony D and Pendleton, Brian J and Roweth, Duncan},
  journal={Physics letters B},
  volume={195},
  number={2},
  pages={216--222},
  year={1987},
  publisher={Elsevier}
}

@article{zheng2024ensemble,
  title={Ensemble {K}alman diffusion guidance: A derivative-free method for inverse problems},
  author={Zheng, Hongkai and Chu, Wenda and Wang, Austin and Kovachki, Nikola and Baptista, Ricardo and Yue, Yisong},
  journal={Transactions on Machine Learning Research},
  year={2025}
}

@article{andrieu2008tutorial,
  title={A tutorial on adaptive {MCMC}},
  author={Andrieu, Christophe and Thoms, Johannes},
  journal={Statistics and computing},
  volume={18},
  number={4},
  pages={343--373},
  year={2008},
  publisher={Springer}
}

@book{sanz2023inverse,
  title={Inverse problems and data assimilation},
  author={Sanz-Alonso, Daniel and Stuart, Andrew and Taeb, Armeen},
  volume={107},
  year={2023},
  publisher={Cambridge University Press}
}

@article{kerrigan2023functional,
  title={Functional flow matching},
  author={Kerrigan, Gavin and Migliorini, Giosue and Smyth, Padhraic},
  journal={arXiv preprint arXiv:2305.17209},
  year={2023}
}

@inproceedings{li2020fourier,
  title={Fourier neural operator for parametric partial differential equations},
  author={Li, Zongyi and Kovachki, Nikola and Azizzadenesheli, Kamyar and Liu, Burigede and Bhattacharya, Kaushik and Stuart, Andrew and Anandkumar, Anima},
  booktitle={International Conference on Learning Representations},
  year={2021}
}

@inproceedings{murray2010elliptical,
  title={Elliptical slice sampling},
  author={Murray, Iain and Adams, Ryan and MacKay, David},
  booktitle={Proceedings of the thirteenth international conference on artificial intelligence and statistics},
  pages={541--548},
  year={2010},
  organization={JMLR Workshop and Conference Proceedings}
}

@article{bellchambers2025exploiting,
  title={Exploiting the Exact Denoising Posterior Score in Training-Free Guidance of Diffusion Models},
  author={Bellchambers, Gregory},
  journal={arXiv preprint arXiv:2506.13614},
  year={2025}
}

@inproceedings{zheng2025inversebench,
  title={Inverse{B}ench: Benchmarking plug-and-play diffusion priors for inverse problems in physical sciences},
  author={Zheng, Hongkai and Chu, Wenda and Zhang, Bingliang and Wu, Zihui and Wang, Austin and Feng, Berthy and Zou, Caifeng and Sun, Yu and Kovachki, Nikola and Ross, Zachary and others},
  booktitle={International Conference on Learning Representations},
  volume={2025},
  pages={90912--90940},
  year={2025}
}

@article{schillings2017analysis,
  title={Analysis of the ensemble {K}alman filter for inverse problems},
  author={Schillings, Claudia and Stuart, Andrew M},
  journal={SIAM Journal on Numerical Analysis},
  volume={55},
  number={3},
  pages={1264--1290},
  year={2017},
  publisher={SIAM}
}

@article{evensen2003ensemble,
  title={The ensemble {K}alman filter: Theoretical formulation and practical implementation},
  author={Evensen, Geir},
  journal={Ocean dynamics},
  volume={53},
  number={4},
  pages={343--367},
  year={2003},
  publisher={Springer}
}

@inproceedings{song2020score,
  title={Score-based generative modeling through stochastic differential equations},
  author={Song, Yang and Sohl-Dickstein, Jascha and Kingma, Diederik P and Kumar, Abhishek and Ermon, Stefano and Poole, Ben},
  booktitle={International Conference on Learning Representations},
  year={2021}
}

@inproceedings{lipman2022flow,
  title={Flow matching for generative modeling},
  author={Lipman, Yaron and Chen, Ricky TQ and Ben-Hamu, Heli and Nickel, Maximilian and Le, Matt},
  booktitle={International Conference on Learning Representations},
  year={2023}
}

@article{goodman2010ensemble,
  title={Ensemble samplers with affine invariance},
  author={Goodman, Jonathan and Weare, Jonathan},
  journal={Communications in applied mathematics and computational science},
  volume={5},
  number={1},
  pages={65--80},
  year={2010},
  publisher={Mathematical Sciences Publishers}
}

@article{yao2026guided,
  title={Guided diffusion sampling on function spaces with applications to {PDE}s},
  author={Yao, Jiachen and Mammadov, Abbas and Berner, Julius and Kerrigan, Gavin and Ye, Jong Chul and Azizzadenesheli, Kamyar and Anandkumar, Animashree},
  journal={Advances in Neural Information Processing Systems},
  volume={38},
  pages={127057--127094},
  year={2026}
}

\appendix

\section{Stochastic interpolants}\label{app:stochastic-interpolants}
Given measures $p_0$ and $p_1$, a stochastic interpolant is defined by the stochastic process
\begin{align}\label{eq:stochastic-interpolant}
    x_t = \alpha_t x_0 + \beta_t x_1,
\end{align}
where $x_0 \sim p_0$, $x_1 \sim p_1$, and $\alpha_t, \beta_t$ for $t \in [0, 1]$ are arbitrary  monotonic $C^2$ functions satisfying the endpoint conditions $\alpha_0 = \beta_1 = 1$ and $\alpha_1 = \beta_0 = 0$. By a slight extension of \cite[Corollary 18]{albergo2025stochastic}, for any positive definite $\{C_t\}_{t \in [0,1]}$, we have that the SDE 
\begin{align}\label{eq:backward-sde}
    \mathrm{d} x_t = \left(b(t, x_t) - \frac{1}{2} C_t s(t, x_t)\right) \mathrm{d} t + \sqrt{C_t} \mathrm{d}W_t,
\end{align}
solved backwards in time from $x_1 \sim p_1$ shares the same marginal law $p(x_t)$ as the forward interpolant (\ref{eq:stochastic-interpolant}). Here, $b(t, x_t) := \mathbb{E}[\dot{x}_t | x_t]$ is the marginal velocity (as used in flow matching) and $s(t, x_t) := \nabla_{x_t} \log p(x_t)$ is the score function. While both $b(t, x_t)$ and $s(t, x_t)$ are intractable, these can be approximated by neural networks $b_\theta(t, x_t)$, $s_\theta(t, x_t)$, respectively, by minimizing the losses
\begin{align}
    \mathcal{L}_b(\theta)&= \mathbb{E}_{x_0 \sim p_0, x_1 \sim p_1, t\sim\mathcal{U}([0,1])}\Big[\|b_\theta(t, x_t) - 
    (\dot{\alpha}_t x_0 + \dot{\beta}_t x_1)\|^2\Big],  \label{eq:fm-objective}\\
    \mathcal{L}_s(\theta) &= \mathbb{E}_{x_0 \sim p_0, x_1 \sim p_1, t\sim\mathcal{U}([0,1])}\Big[\|s_\theta(t, x_t) - 
    \nabla_{x_t} \log p(x_t \mid x_0)\|^2\Big].
    \label{eq:sm-objective}
\end{align}

\subsection{Three-sides of the same coin}
In the special case when the reference is given by a Gaussian $p_1 = \mathcal{N}(0, C)$\footnote{We consider a general covariance $C$ here to adapt to the functional flow matching (FFM) architecture \cite{kerrigan2023functional} used in our Navier-Stokes experiments. FFM uses a Mat\'ern GP covariance instead of white noise.}, we can show that the drift $b(t, x_t)$, the score $s(t, x_t)$ and the clean sample estimate $\mathbb{E}[x_0 \mid x_t]$ are all related linearly. First, we can easily check the relation between $s(t, x_t)$ and $\mathbb{E}[x_0 \mid x_t]$, noting that $p(x_t \mid x_0) = \mathcal{N}(x_t \mid \alpha_t x_0, \beta_t^2 C)$:
\begin{align}
    s(t, x_t) &:= \nabla_{x_t} \log p(x_t) \\
    &= \int \nabla_{x_t} \log p(x_t \mid x_0) p(x_0 \mid x_t) \mathrm{d} x_0 \\
    &= -\beta_t^{-2} C^{-1} \big(x_t - \alpha_t \mathbb{E}[x_0 \mid x_t]\big). \label{eq:score-explicit}
\end{align}

Re-arranging this expression gives us Tweedie's formula:
\begin{align}\label{eq:tweedie-formula}
    \mathbb{E}[x_0 \mid x_t] = \frac{x_t + \beta_t^2 C s(t, x_t)}{\alpha_t}
\end{align}

On the other hand, using the fact that $x_t = \alpha_t x_0 + \beta_t x_1$ implies $x_t = \alpha_t \mathbb{E}[x_0 \mid x_t] + \beta_t \mathbb{E}[x_1 \mid x_t]$ and $b(t, x_t) := \mathbb{E}[\dot{x}_t \mid x_t] = \dot{\alpha}_t \mathbb{E}[x_0 \mid x_t] + \dot{\beta}_t \mathbb{E}[x_1 \mid x_t]$, a simple algebra yields the relation between the drift $b(t, x_t)$ and the Tweedie estimate $\mathbb{E}[x_0 \mid x_t]$:
\begin{align}\label{eq:xo|xt-formula}
    \mathbb{E}[x_0 \mid x_t] = \frac{\beta_t b(t, x_t) - \dot{\beta}_t x_t}{\dot{\alpha}_t \beta_t - \alpha_t \dot{\beta}_t}.
\end{align}
Finally, combining \eqref{eq:tweedie-formula} and \eqref{eq:xo|xt-formula} gives us the relation between the drift $b(t, x)$ and the score $s(t, x_t)$:
\begin{align}\label{eq:drift-score-relation}
    Cs(t, x_t) &= \frac{\alpha_t b(t, x_t) - \dot{\alpha}_t x_t}{\beta_t(\dot{\alpha}_t\beta_t - \alpha_t\dot{\beta}_t)}.
\end{align}

\subsection{Choice of SDE to sample from $p(x_0 \mid x_t)$}\label{eq:sde-choice}
Given particular choices of $C_t$, we can show that the SDE \eqref{eq:backward-sde} not only samples the interpolant marginals $p(x_t)$ {\em but also the conditional $p(x_0 \mid x_t)$}, which is what we need in the denoising step of diffusion waltz. This follows from the following result.

\begin{proposition}
\label{prop:choice-of-diffusion}
    Let $p_1 = \mathcal{N}(0, C)$ for some positive definite $C$. Then, setting $C_t := \sigma_t^2 C$ in \eqref{eq:backward-sde} with $\{\sigma_t\}_{t \in (0, 1)}$ satisfying
    \begin{align}\label{eq:consistent-noise-schedule}
        \frac{\sigma_t^2}{2} =\frac{\beta_t (\alpha_t \dot{\beta}_t - \dot{\alpha}_t \beta_t)}{\alpha_t}, \quad t \in (0, 1),
    \end{align}
    the terminal law of the reverse-time SDE (\eqref{eq:backward-sde}) integrated backwards from time $t$ to $0$ from initial condition $x_t$, is equivalent to the conditional law $p(x_0 \mid x_t)$ of the interpolant.
\end{proposition}
\begin{proof}
    The proof is a slight adaptation of \cite[Proposition C.1]{potaptchik2026meta}.
\end{proof}

Using this result, we can further simplify the expression for our denoising SDE \eqref{eq:backward-sde}. First, by the relation \eqref{eq:drift-score-relation}, the drift of \eqref{eq:backward-sde} reads
\begin{align}
    b(t, x_t) - \frac{\sigma_t^2}{2} C s(t, x_t) &= b(t, x_t) - \frac{\sigma_t^2}{2} \frac{\alpha_t b(t, x_t) - \dot{\alpha}_t x_t}{\beta_t(\dot{\alpha}_t\beta_t - \alpha_t\dot{\beta}_t)} \\
    &\stackrel{\eqref{eq:consistent-noise-schedule}}{=}b(t, x_t) + \left(b(t, x_t) - \frac{\dot{\alpha}_t}{\alpha_t}x_t\right) \\
    &= 2b(t, x_t) - \frac{\dot{\alpha}_t}{\alpha_t}x_t
\end{align}
or equivalently,
\begin{align}
    b(t, x_t) - \frac{\sigma_t^2}{2} C s(t, x_t) = \frac{\dot{\alpha}_t}{\alpha_t}x_t - \sigma_t^2 Cs(t, x_t).
\end{align}

Thus, the denoising SDE \eqref{eq:backward-sde} simplifies to
\begin{align}\label{eq:consistent-backward-sde-b}
    \mathrm{d} x_t = \left(2b(t, x_t) - \frac{\dot{\alpha}_t}{\alpha_t}x_t\right) \mathrm{d} t + \sigma_t \mathrm{d}W_t^C,
\end{align}
or equivalently, using the score instead, we have
\begin{align}\label{eq:consistent-backward-sde-s}
    \mathrm{d} x_t = \left(\frac{\dot{\alpha}_t}{\alpha_t}x_t - \sigma_t^2 Cs(t, x_t)\right) \mathrm{d} t + \sigma_t \mathrm{d}W_t^C,
\end{align}
where $W_t^C := \sqrt{C} W_t$ is the $C$-Brownian motion.

\section{Acceptance probability computations}
\label{app:acceptance-probabilities}

Given a proposal kernel $q(x_0' \mid x_0)$, the Metropolis-Hastings accept/reject probability to sample from the posterior $p(x_0' \mid y)$ is given by
\begin{align}\label{eq:mh-prob}
    \alpha = \min\left\{1, \dfrac{p(x_0' \mid y)q(x_0 \mid x_0')}{p(x_0 \mid y)q(x_0' \mid x_0)}\right\}.
\end{align}
In this appendix, we derive the Metropolis/Hastings acceptance probabilities for diffusion waltz as it appears in Algorithms \ref{alg:random-walk-waltz} and \ref{alg:conditional-noising-waltz}.

\subsection{RWMH diffusion waltz}\label{app:rwmh-acceptance}

We first show that for RWMH diffusion waltz, as described in Algorithm \ref{alg:random-walk-waltz}, this acceptance probability reduces to a simple likelihood ratio, avoiding the need to evaluate prior densities. To show this, we use the following result.

\begin{lemma}
\label{lemma:reversibility}
  The RWMH diffusion waltz proposal kernel $q_\theta(x_0' \mid x_0)$ is reversible with respect to the diffusion prior $p_\theta(x_0)$, i.e.,
  \begin{align}\label{eq:reversibility}
      q_\theta(x_0' | x_0) p_\theta(x_0) = q_\theta(x_0 | x_0') p_\theta(x_0')
  \end{align}
\end{lemma}
\begin{proof}
    The RWMH diffusion waltz proposal kernel (steps 1\&2 of Algorithm \ref{alg:random-walk-waltz}) can be expressed as
    \begin{align}
        q_\theta(x_0' \mid x_0) := \int p_\theta(x_0' \mid x_t) p(x_t \mid x_0) \mathrm{d}x_t.
    \end{align}
    Thus, we have
    \begin{align}
        q_\theta(x_0' | x_0) p_\theta(x_0) &= \int p_\theta(x_0' \mid x_t) p(x_t \mid x_0) p_\theta(x_0) \mathrm{d}x_t \\
        &=  \int \frac{p(x_t \mid x_0')p_\theta(x_0')}{p_\theta(x_t)} p(x_t \mid x_0) p_\theta(x_0) \mathrm{d}x_t \\
        &= \int p(x_t \mid x_0')p_\theta(x_0') \frac{p(x_t \mid x_0) p_\theta(x_0)}{p_\theta(x_t)} \mathrm{d}x_t \\
        &= \int p_\theta(x_0 \mid x_t) p(x_t \mid x_0')p_\theta (x_0') \mathrm{d}x_t \\
        &= q_\theta(x_0 \mid x_0') p_\theta(x_0'),
    \end{align}
    which establishes the reversibility of the diffusion waltz proposal kernel with respect to $p_\theta$.
\end{proof}

From this, we can simplify the ratio in \eqref{eq:mh-prob} as
\begin{align}\label{eq:rwmh-waltz-acceptance}
    \frac{p_\theta(x_0' \mid y)q_\theta(x_0 \mid x_0')}{p_\theta(x_0 \mid y)q_\theta(x_0' \mid x_0)} = \frac{p(y \mid x_0') p_\theta(x_0') q_\theta(x_0 \mid x_0')}{p(y \mid x_0) p_\theta(x_0)q_\theta(x_0' \mid x_0)} = \frac{p(y \mid x_0')}{p(y \mid x_0)},
\end{align}
proving our claim, where the first equality follows from Bayes' rule and the second follows from the reversibility property \eqref{eq:reversibility}.

\subsection{Conditional noising waltz}\label{app:cn-waltz-acceptance}
The conditional waltz kernel is given in full by:
\begin{align}
    &K_\theta(x_0 \rightarrow x_0') = P_{\text{accept}}(x_0 \rightarrow x_0') + P_{\text{reject}}(x_0 \rightarrow x_0')\delta(x_0 - x_0').
\end{align}
where
\begin{align}
    P_{\text{accept}}(x_0  \rightarrow x_0') &:= \int p_\theta(x_0' \mid x_t)q(x_t \mid x_0, y) \min\left\{1, \dfrac{p(y \mid x_0')\,q(x_t \mid x_0', y)\,p(x_t \mid x_0)}{p(y \mid x_0)\,q(x_t \mid x_0, y)\,p(x_t \mid x_0')}\right\} \diff x_t, \\
    P_{\text{reject}}(x_0 \rightarrow x_0') &:= 1 - P_{\text{accept}}(x_0 \rightarrow x_0').
\end{align}
We show that this kernel is reversible with respect to the posterior $p_\theta(x_0 \mid y)$, that is,
\begin{align}\label{eq:cn-kernel-reversible}
    K_\theta(x_0 \rightarrow x_0')p_\theta(x_0 \mid y) = K_\theta(x_0' \rightarrow x_0)p_\theta(x_0' \mid y).
\end{align}

To show this, we use the following lemma
\begin{lemma}
\label{lemma:pointwise-detailed-balance}
    The following identity holds for any $(x_0, x_0', x_t)$:
    \begin{align}\label{eq:pointwise-detailed-balance}
        \begin{split}
        &\left(p_\theta(x_0' \mid x_t)q(x_t \mid x_0, y) \min\left\{1, \dfrac{p(y \mid x_0')\,q(x_t \mid x_0', y)\,p(x_t \mid x_0)}{p(y \mid x_0)\,q(x_t \mid x_0, y)\,p(x_t \mid x_0')}\right\}\right) p_\theta(x_0 \mid y) \\
        &\quad = \left(p_\theta(x_0 \mid x_t)q(x_t \mid x_0', y) \min\left\{1, \dfrac{p(y \mid x_0)\,q(x_t \mid x_0, y)\,p(x_t \mid x_0')}{p(y \mid x_0')\,q(x_t \mid x_0', y)\,p(x_t \mid x_0)}\right\}\right) p_\theta(x_0' \mid y).
        \end{split}
    \end{align}
\end{lemma}
\begin{proof}
    For simplicity, let
    \begin{align}\label{eq:r1}
        r := \dfrac{p(y \mid x_0')\,q(x_t \mid x_0', y)\,p(x_t \mid x_0)}{p(y \mid x_0)\,q(x_t \mid x_0, y)\,p(x_t \mid x_0')}.
    \end{align}
    When $r < 1$, the left hand side and right hand side of \eqref{eq:pointwise-detailed-balance} reads, respectively,
    \begin{align}
        \text{LHS} &= p_\theta(x_0' \mid x_t)q(x_t \mid x_0, y) p_\theta(x_0 \mid y) \, r \\
        \text{RHS} &= p_\theta(x_0 \mid x_t)q(x_t \mid x_0', y) p_\theta(x_0' \mid y).
    \end{align}
    Thus, for equality \eqref{eq:pointwise-detailed-balance} to hold, we need
    \begin{align}\label{eq:r2}
        r = \frac{p_\theta(x_0 \mid x_t)q(x_t \mid x_0', y) p_\theta(x_0' \mid y)}{p_\theta(x_0' \mid x_t)q(x_t \mid x_0, y) p_\theta(x_0 \mid y)}.
    \end{align}
    Now, we have
    \begin{align}
    p_\theta(x_0 \mid x_t) = \frac{p(x_t \mid x_0) p_\theta(x_0)}{p_\theta(x_t)}, \qquad p_\theta(x_0' \mid y) = \frac{p(y \mid x_0) p_\theta(x_0)}{p_\theta(y)},
    \end{align}
    by Bayes' rule and substituting this into \eqref{eq:r2}, we get
    \begin{align}
        \eqref{eq:r2} &= \frac{p(x_t \mid x_0) p_\theta(x_0) \cancel{p_\theta(x_t)}}{p(x_t \mid x_0') p_\theta(x_0') \cancel{p_\theta(x_t)}} \times \frac{q(x_t \mid x_0', y)}{q(x_t \mid x_0, y)} \times 
        \frac{p(y \mid x_0') p_\theta(x_0') \cancel{p_\theta(y)}}{p(y \mid x_0) p_\theta(x_0) \cancel{p_\theta(y)}} \\
        &=  \frac{p(x_t \mid x_0) q(x_t \mid x_0', y) p(y \mid x_0') \cancel{p_\theta(x_0) p_\theta(x_0')}}{p(x_t \mid x_0') q(x_t \mid x_0, y) p(y \mid x_0) \cancel{p_\theta(x_0') p_\theta(x_0)}},
    \end{align}
    which agrees with the expression in \eqref{eq:r1}, and therefore the equality \eqref{eq:pointwise-detailed-balance} holds. By symmetry, \eqref{eq:pointwise-detailed-balance} also holds in the case $r \geq 1$.
\end{proof}

By Lemma \ref{lemma:pointwise-detailed-balance}, since identity \eqref{eq:pointwise-detailed-balance} holds pointwise in $x_t$, it also holds when we integrate this out, which gives us
\begin{align}
    P_{\text{accept}}(x_0 \rightarrow x_0') p_\theta(x_0 \mid y) = P_{\text{accept}}(x_0' \rightarrow x_0) p_\theta(x_0' \mid y).
\end{align}
Trivially, we also have
\begin{align}
    P_{\text{reject}}(x_0 \rightarrow x_0') \delta(x_0 - x_0') p_\theta(x_0 \mid y) = P_{\text{reject}}(x_0' \rightarrow x_0) \delta(x_0' - x_0) p_\theta(x_0' \mid y),
\end{align}
thus combining this, we have that the conditional waltz kernel $K_\theta$ is reversible with respect to $p_\theta(x_0 \mid y)$, i.e., \eqref{eq:cn-kernel-reversible} holds.

\begin{remark}
As a sanity check, the RWMH waltz acceptance probability \eqref{eq:rwmh-waltz-acceptance} follows from the conditional noising waltz acceptance probability \eqref{eq:r1}, by setting $q(x_t \mid x_0, y) = p(x_t \mid x_0)$.
\end{remark}

\section{Statistical linearisation and ensemble Kalman update rules}\label{app:ensemble-kalman}
Here, we provide details on the statistical linearisation procedure to derive the ensemble Kalman conditional noising (EKCN) kernel. The essence of statistical linearisation is to seek a linear approximation to a nonlinear function $\mathcal{A}(x) \approx Ax + b$ that is a best fit in terms of least squares \citep{zheng2024ensemble}, i.e., given an ensemble $\{x^{(j)}\}_{j=1}^J$
\begin{align}
    (A, b) = \argmin_{A, b}\left\{\frac{1}{J}\sum_{j=1}^J\big\|\mathcal{A}(x^{(j)}) - \big(Ax^{(j)} + b\big)\big\|^2\right\}
\end{align}
This gives us
\begin{align}
    A := (C^{xy})^\top (C^{xx})^{-1}, \qquad b := \overline{\mathcal{A}(x)} - A\bar{x},
\end{align}
where
\begin{align}
    &\bar{x} := \frac{1}{J}\sum_{j=1}^J x^{(j)}, \qquad \overline{\mathcal{A}(x)} := \frac{1}{J}\sum_{j=1}^J \mathcal{A}(x^{(j)})\\
    &C^{xx} := \frac{1}{J}\sum_{i,j=1}^J \big(x^{(i)} - \bar{x}\big)\big(x^{(j)} - \bar{x}\big)^\top \qquad C^{xy} := \frac{1}{J}\sum_{i,j=1}^J \big(x^{(i)} - \bar{x}\big)\big(\mathcal{A}(x^{(j)}) - \overline{\mathcal{A}(x)}\big)^\top. \label{eq:empirical-covariances}
\end{align}

A consequence of this is that we can approximate the Jacobian of $\mathcal{A}$ as $D\mathcal{A}(x) \approx (C^{xy})^\top (C^{xx})^{-1}$, which can be computed without explicitly taking derivatives. Thus, we can use this to solve optimisation problems in a derivative-free manner. For example, if we consider the objective
\begin{align}
    U(x) = \frac{1}{2} \|y - \mathcal{A}(x)\|_\Gamma^2,
\end{align}
where $\|x\|_A := x^\top A^{-1} x$,
its gradient can be approximated as
\begin{align}
    \nabla_x U(x) = D^\top \mathcal{A}(x) \Gamma^{-1} \big(\mathcal{A}(x) - y\big) \approx (C^{xx})^{-1} C^{xy} \Gamma^{-1} \big(\mathcal{A}(x) - y\big),
\end{align}
and we can consider gradient descent with this approximation. We can also derive analogues of  more sophisticated solvers such as Levenberg-Marquardt, as we show next.

\subsection{Ensemble Kalman update rule and the Levenberg-Marquardt solver}\label{app:ensemble-kalman}

We closely follow the presentation in \cite[Chapter 13]{sanz2023inverse} for the derivations of the ensemble Kalman variant of the Levenberg-Marquardt (LM) solver. The derivation largely follow from the following simple lemma:

\begin{lemma}[\cite{sanz2023inverse}, Lemma 13.3]\label{lemma:kalman-update}
The following identity holds
\begin{equation}\label{eq:least-squares}
\frac{1}{2}\|y - Ax\|_{\Gamma}^2 + \frac{1}{2}\|x - m\|_{C}^2 = \frac{1}{2}\|x - \mu\|_{\Sigma}^2 + \beta,
\end{equation}
where $\beta$ is a constant independent of $x$, and
\begin{align}
\mu &:= m + K(y - Am), \label{eq:kalman-mu}\\
\Sigma &:= (I - KA)C, \label{eq:kalman-sigma}
\end{align}
where $K$ is the {\bf Kalman gain matrix}, given by
\begin{equation}
K = C A^\top (ACA^\top + \Gamma)^{-1}.
\end{equation}
\end{lemma}

Equations \eqref{eq:kalman-mu}--\eqref{eq:kalman-sigma} define the {\bf Kalman update rule}, which give us closed form expressions for the solution to the regularised least squares problem, given by the LHS of \eqref{eq:least-squares}, and the associated curvature at that minima. 

The Levenberg-Marquardt algorithm seeks to minimise the data misfit loss
\begin{align}\label{eq:data-misfit}
    \mathcal{J}(x) = \frac{1}{2}\|y - \mathcal{A}(x)\|_{\Gamma}^2,
\end{align}
and for a learning rate $h>0$, proceeds by iterating
\begin{align}\label{eq:lm-iteration}
    x_{n+1} = x_n + g_n(h),
\end{align}
where $g_n(h)$ is a minima of the least squares problem
\begin{align}\label{eq:lm-least-squares}
    g_n(h) = \argmin_v\left\{ \frac{1}{2}\big\|y - \mathcal{A}(x_n) - D\mathcal{A}(x_n)  v\big\|_{\Gamma}^2 + \frac{1}{2h}\big\|v\big\|_{C}^2\right\},
\end{align}
for some PSD matrix $C$. This may be interpreted as minimising $\frac12 \big|y - \mathcal{A}(x_n) - D\mathcal{A}(x_n)  v\big|_{\Gamma}^2$, while restricting the search space of $v$ to a ball $\{v : \|v\|_{C} \leq \delta\}$ for some $\delta$ associated to $h$, i.e., minimise within a {\em trust region}.

Invoking Lemma \ref{lemma:kalman-update}, we get
\begin{align}
    g_n(h) = K(h) (y - \mathcal{A}(x_n)),
\end{align}
where the Kalman gain is given by
\begin{align}\label{eq:lm-kalman-gain}
    K(h) := hC D^\top \mathcal{A}(x_n)\left(hD \mathcal{A}(x_n) \,C\, D^\top \mathcal{A}(x_n) + \Gamma \right)^{-1}.
\end{align}

Now recall that by statistical linearisation, given a cloud of particles $\{x^{(j)}\}_{j=1}^J$, we have that
\begin{align}
    \mathcal{A}(x) &\approx (C^{xy})^\top (C^{xx})^{-1}(x - \bar{x}) + \overline{\mathcal{A}(x)}, \label{eq:A-stat-lin}\\
    D\mathcal{A}(x) &\approx (C^{xy})^\top (C^{xx})^{-1}. \label{eq:DA-stat-lin}
\end{align}
Since $C$ in \eqref{eq:lm-least-squares} is arbitrary, we set $C = C^{xx}$. This implies
\begin{align}
    C^{xy} \stackrel{\eqref{eq:empirical-covariances}}{=} CD^\top\mathcal{A}(x),
\end{align}
and
\begin{align}
    C^{yy} &:= \frac{1}{J}\sum_{i,j=1}^J \big(\mathcal{A}(x^{(i)}) - \overline{\mathcal{A}(x)}\big)\big(\mathcal{A}(x^{(j)}) - \overline{\mathcal{A}(x)}\big)^\top \\
    &\stackrel{\eqref{eq:A-stat-lin}}{=} \frac{1}{J}\sum_{i,j=1}^J (C^{xy})^\top (C^{xx})^{-1} (x^{(i)} - \bar{x})(x^{(j)} - \bar{x})^\top (C^{xx})^{-1} C^{xy} \\
    &\stackrel{\eqref{eq:empirical-covariances}}{=} (C^{xy})^\top (C^{xx})^{-1} C_{xx} (C^{xx})^{-1} C^{xy} \\
    &\stackrel{\eqref{eq:DA-stat-lin}}{\approx} D\mathcal{A}(x)CD^\top\mathcal{A}(x).
\end{align}
Thus, we can approximate the Kalman gain \eqref{eq:lm-kalman-gain} in a gradient-free manner using ensembles as
\begin{align}\label{eq:lm-kalman-gain-ensemble}
    K(h) \approx h C^{xy} \left(hC^{yy} + \Gamma\right)^{-1}.
\end{align}
In practice, one iterates \eqref{eq:lm-iteration} for an ensemble of particles $\{x_n^{(j)}\}_{j=1}^J$ in parallel and at each time step $n$, the Kalman gain \eqref{eq:lm-kalman-gain-ensemble} is computed using the ensemble $\{x_n^{(j)}\}_{j=1}^J$.

\subsection{Details on the EKCN kernel}\label{app:ekcn-kernel}
\begin{figure}[t]
    \centering
    \includegraphics[width=1.0\textwidth]{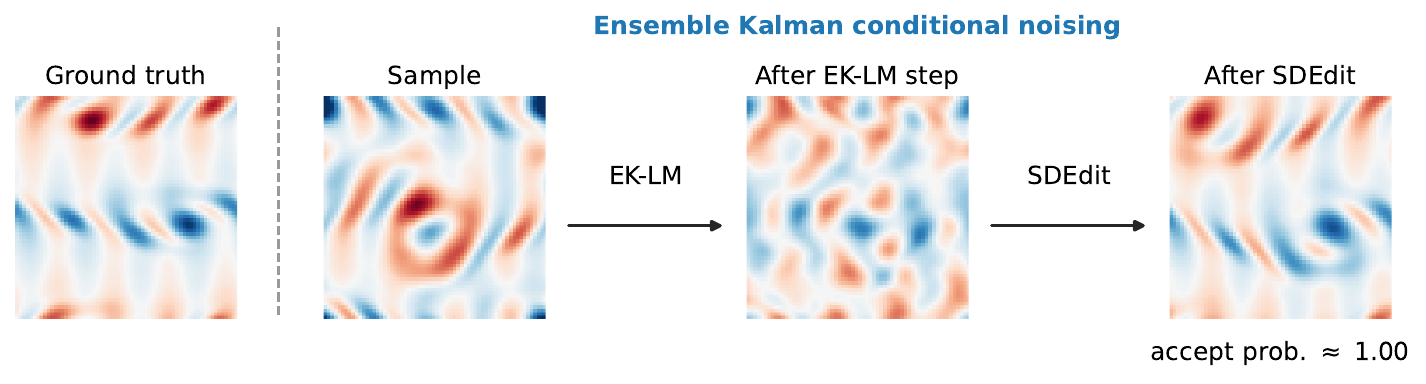}
    \caption{Illustration of the ensemble Kalman conditional noising diffusion waltz (DW-EKCN). This applies a single step of ensemble Kalman Levenberg-Marquardt (EK-LM) update to minimise data misfit, followed by applying SDEdit to project it back onto the data manifold. Earlier in the iterations, acceptance rate is high and convergence is fast as most proposed samples have lower data misfit.}
    \label{fig:ekcn-illustration}
\end{figure}

We use the ensemble Kalman variant of the LM update rule to define our ensemble Kalman conditional noising (EKCN) kernel $q(x_t \mid x_0, y)$. We split this into two steps:
\begin{enumerate}
\item For fixed $h>0$, apply a single ensemble Kalman LM update to $x_0$:
\begin{align}\label{eq:eklm-update}
    \tilde{x}_0 = x_0 + K(h) (y - \mathcal{A}(x_0)),
\end{align}
where $K(h)$ is computed using \eqref{eq:lm-kalman-gain-ensemble}.
\item Add noise via the interpolant
\begin{align}
    x_t = \alpha_t \tilde{x}_0 + \beta_t x_1, \quad x_1 \sim p_1.
\end{align}
\end{enumerate}
Combining this with the denoising step in diffusion waltz
\begin{enumerate}[start=3]
  \item Denoise $x_0' \sim p_\theta(x_0 \mid x_t)$,
\end{enumerate}
the full conditional noising waltz proposal kernel $q(x_0' \mid x_0)$ may be understood as first applying an ensemble Kalman update to minimise the data misfit \eqref{eq:data-misfit} (step 1); however, since without prior regularisation, this is likely to push the particle off-manifold, this is corrected by applying SDEdit \cite{mengsdedit} (steps 2-3), which keeps it on-manifold (see Figure \ref{fig:ekcn-illustration} for an illustration).

In practice, we run $J$ MCMC chains in parallel and the Kalman gain $K(h)$ at each step is computed using the ensemble. However, when the observation dimension is high, the matrix inverse in \eqref{eq:lm-kalman-gain-ensemble} is expensive to compute. Following a standard trick used in ensemble Kalman filters \cite{evensen2003ensemble}, we use the pushthrough identity identity $V (V^T V + \lambda I)^{-1} = (V V^T + \lambda I)^{-1} V$ to reduce this complexity.
Letting $\Gamma = \sigma_y^2 I$, $C^{xy} = U^\top V$ and $C^{yy} = V^\top V$, where
\begin{align}
    U_j = \frac{1}{\sqrt{J}}(x_0^{(j)} - \bar{x}_0), \qquad V_j = \frac{1}{\sqrt{J}}(\mathcal{A}(x_0^{(j)}) - \overline{\mathcal{A}(x_0)}),
\end{align}
we have
\begin{align}
    K(h) &\stackrel{\eqref{eq:lm-kalman-gain-ensemble}}{\approx} h U^\top V \left(hV^\top V + \sigma_y^2 I\right)^{-1} \\
    &=  U^\top V \left(V^\top V + (\sigma_y^2 / h) I\right)^{-1} \\
    &=  U^\top \left(V V^\top + (\sigma_y^2 / h) I\right)^{-1} V. \label{eq:lm-kalman-gain-ensemble-reduced}
\end{align}
Noting that the matrix inverse in \eqref{eq:lm-kalman-gain-ensemble-reduced} now has size $J \times J$, and since the ensemble size $J$ is a hyperparameter that one can choose freely, the computation of \eqref{eq:lm-kalman-gain-ensemble-reduced} can be made tractable.

We now make an important remark that in the computation of the acceptance probability 
\begin{align}
    \alpha = \min\left\{1, \dfrac{p(y \mid x_0')\,q(x_t \mid x_0', y)\,p(x_t \mid x_0)}{p(y \mid x_0)\,q(x_t \mid x_0, y)\,p(x_t \mid x_0')}\right\},
\end{align}
we {\em use the same Kalman gain $K(h)$ in the computation of $q(x_t \mid x_0', y)$ and $q(x_t \mid x_0, y)$}to ensure that the same kernel $q$ is used in both the numerator and denominator. More precisely, given an ensemble $\{x_0^{(j)}\}_{j=1}^J$, we first compute $K(h)$ using this ensemble and compute
\begin{align}
    q(x_t^{(j)} \mid x_0^{(j)}, y) &= \mathcal{N}\Big(x_t^{(j)} \mid \alpha_t \big(x_0^{(j)} + K(h) \big(y - \mathcal{A}(x_0^{(j)})\big)\big), \beta_t^2 C\Big) \label{eq:cond-noise-forward} \\
    q(x_t^{(j)} \mid x_0^{'(j)}, y) &= \mathcal{N}\Big(x_t^{(j)} \mid \alpha_t \big(x_0^{'(j)} + K(h) \big(y - \mathcal{A}(x_0^{'(j)})\big)\big), \beta_t^2 C\Big), \label{eq:cond-noise-reverse}
\end{align}
for {\em the same $K(h)$}, i.e., we do not recompute it in \eqref{eq:cond-noise-reverse} using the updated samples $\{x_0^{'(j)}\}_{j=1}^J$.

\begin{remark}
    The interacting particle-style update in equations (\ref{eq:cond-noise-forward})–(\ref{eq:cond-noise-reverse}) subtly breaks exactness of the method due to the dependence of the proposal on the law of the particles. Methods exist to overcome this, e.g., splitting particles into groups and updating one group using information from the other \cite{goodman2010ensemble}; however, in practice, we found that simply using all the ensemble members yield faster convergence and better empirical performance.
\end{remark}

\section{Adaptive MCMC tuning}
\label{app:adaptive-tuning}

We consider Robbins-Monro method to adaptively tune the hyperparameters $t \in (0, 1)$ (the noising time used in waltz) and $h \in (0, 1)$ (the ensemble Kalman step size) \citep{andrieu2008tutorial}. This reads
\begin{align}
    \text{logit}(t_{n+1}) &= \mathtt{clip}\big(\text{logit}(t_n) + \eta_t(\alpha_n - \alpha^*)\big), \\
    \text{logit}(h_{n+1}) &= \mathtt{clip}\big(\text{logit}(h_n) + \eta_h(\alpha_n - \alpha^*)\big), 
\end{align}
where $\alpha^*$ is a target acceptance rate, $\alpha_n$ is the instantaneous acceptance rate (i.e., how many samples out of $J$ were accepted at step $n$), and $\eta_t, \eta_h > 0$ are step-sizes, all specified by the users. Clipping is applied e.g. to prevent $t$ from going down to 0, which is important as we discuss later. In our experiments, we use the adaptive tuning hyperparameters as shown in Table \ref{tab:adaptive-hyperparams}.

\begin{table}[htbp]
    \centering
    \caption{Choice of adaptive MCMC hyperparameters}
    \begin{tabular}{ccccccc c}
        \toprule
        $\alpha^*$ & $\eta_t$ & $\eta_h$ & $t_{\text{min}}$ & $t_{\text{max}}$ & $h_{\text{min}}$ & $h_{\text{max}}$ & solver\\
        \midrule
        0.3 & 0.3 & 0.3 & 0.3 & 0.7 & 0.01 & 1.0 & \texttt{LM}\\
        \bottomrule
    \end{tabular}
    \label{tab:adaptive-hyperparams}
\end{table}

In general, we find that adaptively tuning the hyperparameters starting from high $t, h$ improves the performance. Intuitively, this is because the ensemble Kalman LM update \eqref{eq:eklm-update} makes larger jumps early in the iteration, requiring a larger noising time $t$ to project the proposal $\tilde{x}_0$ back onto the data manifold. However, as the chains converge, the ensemble Kalman updates become smaller with proposals $\tilde{x}_0$ being more on-manifold, requiring smaller $t$. Robbins-Monro automatically drives $t$ down to smaller values as MCMC progresses, as acceptance rate decreases when the chains have converged to high likelihood regions and smaller ``edits" become necessary for higher acceptance.
However, decreasing $t$ creates another problem, where the acceptance rate tends to become vanishingly small. To see this, we split the log acceptance probability into a likelihood log-ratio and the remaining term that emerges due to the conditional noising
\begin{align}
    \log \alpha &= \log \left(\dfrac{p(y \mid x_0')\,q(x_t \mid x_0', y)\,p(x_t \mid x_0)}{p(y \mid x_0)\,q(x_t \mid x_0, y)\,p(x_t \mid x_0')}\right) \\ 
    &=  \log \left(\dfrac{p(y \mid x_0')}{p(y \mid x_0)}\right) +  \log \left(\dfrac{q(x_t \mid x_0', y)\,p(x_t \mid x_0)}{q(x_t \mid x_0, y)\,p(x_t \mid x_0')}\right).
\end{align}
The remaining term can be expanded as
\begin{align}
    &\log \left(\dfrac{q(x_t \mid x_0', y)\,p(x_t \mid x_0)}{q(x_t \mid x_0, y)\,p(x_t \mid x_0')}\right) \\
    &= -\frac{1}{2\beta_t^2} \big\|x_t - \alpha_t(x_0' + K(h)(y - A(x_0')))\big\|_C^2 - \frac{1}{2\beta_t^2} \big\|x_t - \alpha_tx_0\big\|_C^2 \nonumber \\
    &\quad +\frac{1}{2\beta_t^2} \big\|x_t - \alpha_t(x_0 + K(h)(y - A(x_0)))\big\|_C^2 + \frac{1}{2\beta_t^2} \big\|x_t - \alpha_tx_0'\big\|_C^2 \\
    &= - \frac{\alpha_t^2}{2\beta_t^2} \big\|K(h)(y - A(x_0'))\big\|_C^2 + \frac{\alpha_t^2}{2\beta_t^2} \big\|K(h)(y - A(x_0))\big\|_C^2 \nonumber \\
    &\quad + \frac{\alpha_t}{\beta_t^2} \big<x_t - \alpha_t x_0', K(h)(y - A(x_0'))\big>_C - \frac{\alpha_t}{\beta_t^2} \big<x_t - \alpha_t x_0, K(h)(y - A(x_0))\big>_C. \label{eq:log-ratio-expression}
\end{align}
Using the fact that $x_t = \alpha_t \big(x_0 + K(h)(y - \mathcal{A}(x_0))\big) + \beta_t x_1$, for some $x_1 \sim \mathcal{N}(0, C)$, we get (key terms that change are highlighted)
\begin{align}
    \eqref{eq:log-ratio-expression} &= - \frac{\alpha_t^2}{2\beta_t^2} \big\|K(h)(y - A(x_0'))\big\|_C^2 + \frac{\alpha_t^2}{2\beta_t^2} \big\|K(h)(y - A(x_0))\big\|_C^2 \nonumber\\
    &\quad + \frac{\alpha_t}{\beta_t^2} \big<x_t - \alpha_t x_0', K(h)(y - A(x_0'))\big>_C \nonumber \\
    &\quad - \frac{\alpha_t}{\beta_t^2} \big<{\color{magenta} \alpha_t K(h)(y - A(x_0)) + \beta_t x_1}, K(h)(y - A(x_0))\big>_C \\
    &= - \frac{\alpha_t^2}{2\beta_t^2} \big\|K(h)(y - A(x_0'))\big\|_C^2  {\color{magenta} - \frac{\alpha_t^2}{2\beta_t^2} \big\|K(h)(y - A(x_0))\big\|_C^2} \nonumber \\
    &\quad - \frac{\alpha_t}{\beta_t^2} \big<x_t - \alpha_t x_0', K(h)(y - A(x_0'))\big>_C + \frac{\alpha_t}{\beta_t} \big<x_1, K(h)(y - A(x_0))\big>_C.
\end{align}
Note that as $t \rightarrow 0$, we have that $\alpha_t / \beta_t \rightarrow \infty$ by our assumptions on $\alpha_t, \beta_t$, which drives the first two terms to $-\infty$ (the last two terms are indefinite; in fact, the last term is zero in expectation, conditioned on $x_0$). For small $t$, this term therefore dominates and drives the acceptance probability close to zero, regardless of how diffuse the likelihood is.

To prevent this collapse, we (i) set a lower bound on $t$ so that it doesn't reach to $0$, and (ii) tune the LM learning rate $h$ jointly with $t$; since $K(h) \rightarrow 0$ as $h \rightarrow 0$, it helps to counterbalance the growth of $\alpha_t / \beta_t$ as $t \rightarrow 0$.
We note that when $h = 0$, we simply get back RWMH waltz. Thus, as EKCN waltz progresses and the chains converge, it starts to behave more like RWMH waltz.

\section{Experimental details}
\label{app:experiment-details}

\subsection{Problem statement}
We consider the Navier-Stokes (NS) equation on a 2D torus (vorticity form)
\begin{equation}
    \begin{aligned}
        \partial_t \omega(\boldsymbol{x}, t) + \boldsymbol{u}(\boldsymbol{x}, t) \cdot \nabla \omega(\boldsymbol{x}, t) &= \nu \Delta \omega(\boldsymbol{x}, t) + f(\boldsymbol{x}), & \boldsymbol{x} &\in (0, 2\pi)^2, t \in (0, T] \\
        \nabla \cdot \boldsymbol{u}(\boldsymbol{x}, t) &= 0, & \boldsymbol{x} &\in (0, 2\pi)^2, t \in [0, T] \\
        \omega(\boldsymbol{x}, 0) &= \omega_0(\boldsymbol{x}), & \boldsymbol{x} &\in (0, 2\pi)^2.
    \end{aligned}
    \label{eq:navier-stokes}
\end{equation}
Here, $\boldsymbol{u}(\boldsymbol{x}, t) := (u(\boldsymbol{x}, t), v(\boldsymbol{x}, t))$ is the velocity field, $\omega(\boldsymbol{x}, t) := \frac{\partial v}{\partial x}(\boldsymbol{x}, t) - \frac{\partial u}{\partial y}(\boldsymbol{x}, t)$ is the vorticity field, $\nu$ is the Reynolds' number and $f(\boldsymbol{x})$ is an external forcing forcing.

The observations for our inverse problem are obtained as follows: First, we solve the NS system for $T$ time units using an in-house pseudo-spectral NS solver with finite-difference time-stepping to obtain $\mathcal{F}_T : \omega_0 \mapsto \omega_T$. Then, we take a spatially strided downsampling $\mathcal{P}_k$ ($\times k$ downsampling), with additive isotropic Gaussian noise scaled relative to the dataset standard deviation $\sigma_0$. Thus the observation model reads
\begin{align}
y = \mathcal{P}_k \circ \mathcal{F}_T(\omega_0) + \sigma_{y}\,\varepsilon, \qquad \varepsilon \sim \mathcal{N}(0, I).
\end{align}

In our experiments, we fix $k = 4$ and consider various choices of $T$ and $\sigma_y$ controlling the degree of nonlinearity and intensity of noise (see Table \ref{tab:regimes})

\begin{table}[htbp]
    \centering
    \caption{Experimental regimes}
    \begin{tabular}{ccc}
        \toprule
         & Low & High \\
        \midrule
        Nonlinearity & $T = 1$ & $T = 5$  \\
        Noise & $\sigma_y / \sigma_0 = 0.1$ & $\sigma_y / \sigma_0 = 1.0$ \\
        \bottomrule
    \end{tabular}
    \label{tab:regimes}
\end{table}

\subsection{Data}
Our data comes from Navier-Stokes vorticity trajectories ($200$ trajectories of $501$ timesteps at $64\times64$ resolution with Kolmogorov forcing $f(x, y)=-4\cos(4y)$ and Reynolds' number $\nu = 40$) as used in \citep{li2022learning}. To avoid highly correlated samples, we thin the trajectories by subsampling every $5$ timesteps, giving us $\sim 20,000$ snapshots, split by trajectory into $180$ train / $20$ validation trajectories.

\subsection{Model architecture}
For our generative prior, we use
functional flow matching \cite{kerrigan2023functional}, where we parameterise the flow matching marginal velocity $b_\theta(t, \omega_t)$ by a Fourier Neural Operator (FNO)~\citep{li2020fourier}. The input field $\omega_t$ is concatenated with an explicit 2D positional grid $\boldsymbol{x}$ and a constant channel broadcasting the flow-matching time $t$, then lifted to a hidden width of $64$ channels via a pointwise MLP. This is followed by $4$ Fourier layers, each retaining the lowest $32$ Fourier modes per spatial axis, with GELU activations. The final hidden representation is projected back to the physical output channel via a second pointwise MLP.

We use a linear interpolant $\alpha_t = 1-t$, $\beta_t = t$ and a Mat\'ern Gaussian process (GP) reference measure $p_1$. In practice, we sample from the GP using the Karhunen-Lo\`eve expansion
\begin{align}\label{eq:karhunen-loeve}
    \omega_1(\boldsymbol{x}) = \sum_{k \in \mathbb{Z}^2} \sqrt{\lambda_k}\, \xi_k \, e_k(\boldsymbol{x}), \qquad \xi_k \sim \mathcal{N}(0,1) \ \text{i.i.d.},
\end{align}
where $e_k$ are the 2D Fourier modes on the two-torus (i.e., $\sin(2\pi k\boldsymbol{x}), \cos(2\pi k\boldsymbol{x})$) and $\lambda_k$ is the spectral density of the kernel:
\begin{align}
\lambda_k = \sigma^2 \, \frac{2\pi (2\nu)^{\nu+1}}{\ell^{2\nu}} \left(\frac{2\nu}{\ell^2} + |k|^2\right)^{-(\nu+1)}, \qquad k \in \mathbb{Z}^2.
\end{align}

We use the denoising SDE \eqref{eq:consistent-backward-sde-b}, which in the linear interpolant setting, reads
\begin{align}\label{eq:ns-denoising-sde}
    \mathrm{d} \omega_t(\boldsymbol{x}) = \left(\frac{\omega_t(\boldsymbol{x})}{1-t} + 2 b_\theta(t, \omega_t(\boldsymbol{x}))\right) \mathrm{d} t + \sqrt{\frac{2t}{1-t}} \mathrm{d}W_t^C,
\end{align}
solved backwards in time, where $C$ is the Mat\'ern covariance operator and $\{W_t^C\}_{t \in [0, 1]}$ is the $C$-Brownian motion. By Proposition \ref{prop:choice-of-diffusion}, this samples from $p_\theta(\omega_0 \mid \omega_t)$. In practice, we discretise the field $\omega_t$ on a grid of size $N \times N$, whose discretisation we denote by $\boldsymbol{\omega}_t$. We solve \eqref{eq:ns-denoising-sde} using Euler-Maruyama:
\begin{align}\label{eq:denoising-sde-discretised}
    \boldsymbol{\omega}_{n+1} = \boldsymbol{\omega}_{n} - \left(\frac{\boldsymbol{\omega}_{n}}{1-t_n} + 2 \boldsymbol{b}_\theta(t, \boldsymbol{\omega}_{n})\right) \Delta t + \sqrt{\frac{2t_n \Delta t}{1-t_n}} \boldsymbol{\xi}_n^C,
\end{align}
where $\boldsymbol{\xi}_n^C$ are i.i.d. draws from the GP using \eqref{eq:karhunen-loeve}. The setting we use in our experiments is shown in Table~\ref{tab:gp-hyperparams}.

\begin{table}[htbp]
    \centering
    \caption{Choice of Mat\'ern GP reference hyperparameters}
    \begin{tabular}{cccc}
        \toprule
        Grid size $N$ & Smoothness $\nu$ & Lengthscale $\ell$ & Variance $\sigma^2$ \\
        \midrule
        64 & 0.5 & 0.1 & 1.0 \\
        \bottomrule
    \end{tabular}
    \label{tab:gp-hyperparams}
\end{table}

\subsection{Metrics}\label{app:metricx}
To evaluate the performance of the posterior sampling methods, we draw $J$ samples $\{x_{0,j}\}_{j=1}^J$ from the approximate posterior $p(x_0 \mid y)$ produced by each method and compare them against the ground truth $x_0^\dagger$. We report the following metrics, each of which captures a different aspect of sample quality: accuracy of the point estimate, distributional fit, and calibration of the uncertainty.

\paragraph{Relative $\ell_2$ (Rel-$\ell_2$).}
The Rel-$\ell_2$ error measures how close the ensemble mean $\bar{x}_0$ is to the ground truth, normalized by the norm of the ground truth itself. It is a standard measure of reconstruction accuracy and is agnostic to the spread of the posterior samples.
\begin{align}
    \mathtt{RelL2} = \frac{\| \bar{x}_0 - x_0^\dagger\|_{L^2}}{\|x_0^\dagger\|_{L^2}}, \quad \text{where} \quad \bar{x}_0 = \frac{1}{J} \sum_{j=1}^J x_{0,j}
\end{align}
Lower values indicate that the posterior mean is closer to the true signal.

\paragraph{Continuous ranked probability score (CRPS).}
Unlike Rel-$\ell_2$, CRPS is a proper scoring rule that evaluates the full ensemble rather than just its mean. It rewards both accuracy (samples close to $x_0^\dagger$) and sharpness (low spread among samples), while penalizing ensembles that are either biased or overly diffuse.
\begin{align}
   \mathtt{CRPS} =  \frac{1}{J}\sum^{J}_{j=1} \| x_0^{(j)} - x_0^\dagger \|_{L^1} -
   \frac{1}{2J(J-1)}\sum^{J}_{i,j=1}\| x_0^{(i)} - x_0^{(j)} \|_{L^1} .
\end{align}
Lower CRPS values indicate a better-calibrated and more accurate ensemble.

\paragraph{Spread skill ratio (SSR).}
The SSR assesses whether the spread of the posterior ensemble is consistent with its actual error, i.e., whether the uncertainty estimates are well calibrated. The numerator captures the average ensemble variance (spread), while the denominator captures the squared error of the ensemble mean (skill).
\begin{align}
    \mathtt{SSR} =\sqrt{\left(\frac{J + 1}{J}\right) \frac{\frac{1}{J}\sum_{j=1}^J \| x_0^{(j)} - \bar{x}_0\|_{L^2}^2}{\|\frac{1}{J}\sum_{j=1}^J x_0^{(j)} - x_0^\dagger\|_{L^2}^2}},
\end{align}
An SSR close to $1$ indicates that the ensemble spread matches the actual estimation error (i.e., the posterior is well calibrated). Values of $\mathtt{SSR} < 1$ indicate an overconfident (underdispersed) posterior, while $\mathtt{SSR} > 1$ indicates an overdispersed posterior.

\subsection{Diffusion waltz details}
We tune both the RWMH and EKCN waltz using Robbins--Monro stochastic approximation (Appendix \ref{app:adaptive-tuning}), with the step-size parameters $t$ and $h$ initialized to $0.7$ and $1.0$, respectively. The ensemble is initialized from diffusion prior samples, and we use $J = 100$ ensemble members throughout. For the EKCN waltz, we use all $100$ members for the computation of the Kalman gain (Appendix \ref{app:ensemble-kalman}).

We run $100$ MCMC chains in parallel for $250$ steps with a burn-in of $249$ steps; that is, we compute all metrics using only the final ensemble members. The underlying denoising SDE~\eqref{eq:ns-denoising-sde} is discretised using the Euler--Maruyama scheme, and we do not apply MADM correction \citep{lam2026metropolis} for exact sampling.

\subsection{Baseline details}\label{app:baselines}

We briefly explain each baselines we use below. The hyperparameters are tuned to minimise the CRPS across five validation samples. As with diffusion waltz, all baselines use $J = 100$ ensembles. Metrics are computed using the final converged ensemble members.

\paragraph{Ensemble Kalman Inversion (EKI) \citep{iglesias2013ensemble, schillings2017analysis}.} This uses a derivative-free optimisation method using ensembles, as shown in \ref{app:ensemble-kalman} to minimise the data-fit loss $\frac12\|y - \mathcal{A}(x)\|^2_\Gamma$. EKI does not take into account the prior information of $x$, thus the optimised $x^*$ may lie outside of the data manifold. In addition, being essentially an optimisation method, the particles are prone to collapsing to the same states at convergence. We use a maximum of $50$ iterations with early stopping applied if the loss improvement is less than \texttt{1e-3}. The solver we use is Levenberg-Marquardt with $h=1$. We consider stochastic perturbations of the observations, as is standard in EnKF/EKI \citep{schillings2017analysis}.

\paragraph{Ensemble Kalman guidance (EnKG) \citep{zheng2024ensemble}.} This frames guidance as predictor-corrector sampling in diffusion models with information about $y$ injected during the corrector step using ensemble Kalman update rules (Appendix \ref{app:ensemble-kalman}). Since the Kalman updates are used only to guide the denoising trajectory of a diffusion model, in contrast to EKI, the obtained samples will be on-manifold. However, it shares the same issue of particles collapsing to the same point at convergence. We use $60$ denoising steps with guidance strength set to $\zeta=0.1$.

\paragraph{Feynman-Kac (FK) Steering \citep{singhal2025general}.}
This uses sequential Monte Carlo (SMC) to condition on observations. We use the gradient-free variant, which runs a bootstrap particle filter over the reverse diffusion process, where particles are resampled according to how much their Tweedie-estimate's fit to the observation has improved. This method is asymptotically exact, yielding the posterior $p(x_0 \mid y)$ exactly as the number of particles $J$ are increased. We use $50$ denoising steps, \texttt{DIFFERENCE} potential and tempering parameter $\zeta = 0.03$.

\paragraph{ESS-Flow \citep{kalaivanan2025ess}.}
This considers the flow matching ODE as a deterministic transport map from the Gaussian reference $p_1$ to the data distribution $p_0$ and applies elliptic slice sampling (ESS)  \cite{murray2010elliptical} using the Gaussian reference prior, with likelihood $p(y \mid T_\theta(x_1))$, where $T_\theta$ is the flow matching ODE flow from $t = 1 \rightarrow 0$. This is also an exact method, up to ODE discretisation error. By default, this uses delayed acceptance, where step sizes are shrunk until a particle is accepted; this leads to many forward model evaluations, making it expensive. We use $25$ denoising ODE steps and $200$ MCMC steps; no multifidelity ODE sampling is considered.

\paragraph{Blade \citep{zheng2025blade}.}
This is based on the split Gibbs approach PnP-DM \cite{wu2024principled}, which introduces auxiliary variables $z$, which are noised versions of the clean data $x_0$ and performs Gibbs sampling to iteratively sample $p(z \mid y) \propto p(y \mid z) p(z \mid x_0)$ (likelihood step) and $p(x_0 \mid z)$ (prior step). The prior step is solved using an exact denoiser \cite{bellchambers2025exploiting} and the likelihood step is performed using unadjusted Langevin algorithm (ULA). In Blade, the gradients in this Langevin step is approximated by statistical linearisation, similar to EKI or EnKG. We use $50$ annealing steps, $50$ Langevin steps and $25$ denoising SDE steps.

\paragraph{MFM-GF \citep{potaptchik2026meta, holderrieth2026diamond}.} Meta flow maps (MFM) / diamond maps learn a one-step distilled model to sample from $p(x_0 \mid x_t)$. Using this, one can rapidly sample from the exact likelihood $p(y \mid x_t) = \mathbb{E}_{x_0 \mid x_t}\left[p(y \mid x_0)\right]$ using Monte-Carlo and compute the likelihood score $\nabla_{x_t} \log \mathbb{E}_{x_0 \mid x_t}\left[p(y \mid x_0)\right]$ by taking the gradients through the Monte-Carlo estimation. In MFM-GF, this gradient can be computed without taking gradients through the likelihood $p(y \mid x_0)$ using importance sampling. We use $100$ denoising steps and for each ensemble member, use $10$ MC samples to approximate the likelihood.

\subsection{Comparison with gradient-based guidance methods}\label{app:gradient-based-approaches}
We also make comparisons to several common gradient-based guidance methods, namely, DPS \citep{chung2022diffusion}, MMPS \citep{rozet2024learning} and DAPS \citep{zhang2025improving}\footnote{Since we use a functional flow matching prior, we use functional modifications to all methods, e.g., our DPS is technically equivalent to FunDPS \citep{yao2026guided} --- we make similar modifications to MMPS and DAPS, as well as all gradient-free methods where necessary.}. 
Since our forward operator $\mathcal{A} = \mathcal{P}_k \circ \mathcal{F}_T$ is non-differentiable, we cannot use these methods directly. Thus, we train FNO surrogates for the Navier-Stokes solver $\mathcal{F}_T$ and use this as low-cost differentiable replacements for the finite difference solver. We display results in Tables \ref{tab:t1-results}--\ref{tab:t5-results} in gray.

We see that DPS and MMPS achieve competitive performance and perform the best in the low nonlinearity / low noise setting.
However, in the other settings (especially high nonlinearity) they are mostly outperformed by the MCMC based methods: DW-EKCN, ESS-flow and Blade (except on $T=5$). Note that both DPS and MMPS only yield biased approximations to the posterior, which become more significant as the forward operator becomes more nonlinear, explaining why it struggles in the high nonlinearity setting. 
DAPS does not perform well on this problem -- it is known to struggle when the forward model is given by a PDE operator \citep{zheng2025inversebench}.

\begin{table}[t]
\centering
\caption{Results in the low nonlinearity regime ($T=1$). Best results in {\bf{bold}}, second best \underline{underlined}. Gradient based methods are displayed in gray.}
\resizebox{\textwidth}{!}{%
\begin{tabular}{l ccc ccc}
\toprule
\textbf{Observation noise} & \multicolumn{3}{c}{Low noise} & \multicolumn{3}{c}{High noise} \\
\cmidrule(lr){2-4} \cmidrule(lr){5-7}
\textbf{Methods} & Rel $\ell_2\downarrow$ & CRPS$\downarrow$ & SSR$\rightarrow 1$ & Rel $\ell_2\downarrow$ & CRPS$\downarrow$ & SSR$\rightarrow 1$ \\
\midrule
\color{gray} DPS  & \color{gray} 0.075 (0.011) & \color{gray} 0.039 (0.008) & \color{gray} 1.393 (0.084) & \color{gray} 0.275 (0.078) & \color{gray} 0.127 (0.023) & \color{gray} 0.855 (0.162) \\
\color{gray} MMPS & \color{gray} \bf{0.064 (0.009)} & \color{gray} \bf{0.032 (0.006)} & \color{gray} \bf{0.983 (0.095)} & \color{gray} 0.254 (0.063) & \color{gray} \underline{0.116 (0.020)} & \color{gray} \underline{0.927 (0.189)} \\
\color{gray} DAPS & \color{gray} 0.401 (0.013) & \color{gray} 0.182 (0.037) & \color{gray} \underline{1.188 (0.105)} & \color{gray} 0.710 (0.062) & \color{gray} 0.338 (0.028) & \color{gray} 1.342 (0.068) \\
EKI  & 0.517 (0.077) & 0.370 (0.086) & 0.024 (0.006) & 0.832 (0.064) & 0.560 (0.042) & 0.126 (0.009) \\
EnKG               & 0.129 (0.014) & 0.086 (0.009) & 0.076 (0.019) & 0.617 (0.120) & 0.414 (0.049) & 0.071 (0.009) \\
FK Steering  & 0.786 (0.109) & 0.538 (0.150) & 0.012 (0.024) & 0.782 (0.109) & 0.500 (0.145) & 0.101 (0.024) \\
ESS-Flow  & 0.178 (0.054) & 0.080 (0.033) & 1.405 (0.198) & \underline{0.248 (0.030)} & 0.118 (0.030) & 1.179 (0.154) \\
Blade & \underline{0.069 (0.012)} & \underline{0.037 (0.007)} & 1.797 (0.139) & 0.296 (0.071) & 0.139 (0.033) & \bf{0.982 (0.203)} \\
MFM-GF  & 0.406 (0.101) & 0.191 (0.084) & 1.395 (0.442) & 0.425 (0.089) & 0.202 (0.077) & 1.406 (0.416) \\
\rowcolor{gray!15}
DW-RWMH  & 0.451 (0.101) & 0.177 (0.068) & 1.703 (0.404) & 0.550 (0.053) & 0.229 (0.055) & 1.476 (0.216) \\
\rowcolor{gray!15}
DW-EKCN  & 0.086 (0.013) & 0.045 (0.010) & 1.396 (0.199) & \bf{0.220 (0.049)} & \bf{0.101 (0.021)} & 1.081 (0.293) \\
\bottomrule
\end{tabular}%
}
\label{tab:t1-results}
\end{table}

\begin{table}[t]
\centering
\caption{Results in the high nonlinearity regime ($T=5$). Best results in {\bf{bold}}, second best \underline{underlined}. Gradient based methods are displayed in gray.}
\resizebox{\textwidth}{!}{%
\begin{tabular}{l ccc ccc}
\toprule
\textbf{Observation noise} & \multicolumn{3}{c}{Low noise} & \multicolumn{3}{c}{High noise} \\
\cmidrule(lr){2-4} \cmidrule(lr){5-7}
\textbf{Methods} & Rel $\ell_2\downarrow$ & CRPS $\downarrow$ & SSR$\rightarrow 1$ & Rel $\ell_2\downarrow$ & CRPS $\downarrow$ & SSR$\rightarrow 1$ \\
\midrule
\color{gray} DPS  & \color{gray} 0.472 (0.084) & \color{gray} \underline{0.201 (0.069)} & \color{gray} 1.834 (0.285) & \color{gray} 0.506 (0.084) & \color{gray} \underline{0.233 (0.079)} & \color{gray} 1.244 (0.250) \\
\color{gray} MMPS & \color{gray} 0.560 (0.090) & \color{gray} 0.276 (0.077) & \color{gray} 1.572 (0.382) & \color{gray} 0.523 (0.070) & \color{gray} 0.241 (0.069) & \color{gray} 1.461 (0.292) \\
\color{gray} DAPS & \color{gray} 0.639 (0.071) & \color{gray} 0.327 (0.071) & \color{gray} 1.510 (0.281) & \color{gray} 0.697 (0.057) & \color{gray} 0.351 (0.075) & \color{gray} 1.358 (0.258) \\
EKI  & 1.059 (0.069) & 0.754 (0.104) & 0.031 (0.006) & 0.865 (0.086) & 0.522 (0.069) & 0.267 (0.041) \\
EnKG               & \underline{0.418 (0.153)} & 0.279 (0.140) & 0.098 (0.035) & 0.812 (0.157) & 0.523 (0.070) & 0.084 (0.006) \\
FK Steering  & 1.383 (0.164) & 1.000 (0.068) & 0.016 (0.020) & 0.991 (0.053) & 0.511 (0.060) & 1.084 (0.127) \\
ESS-Flow  & 0.506 (0.123) & 0.242 (0.104) & \bf{0.987 (0.242)} & 0.491 (0.105) & 0.233 (0.096) & \bf{0.987 (0.230)} \\
Blade & 0.421 (0.309) & 0.239 (0.180) & 2.312 (0.899) & \underline{0.442 (0.180)} & \underline{0.231 (0.142)} & 1.326 (0.637) \\
MFM-GF  & 0.549 (0.087) & 0.262 (0.084) & 1.245 (0.325) & 0.547 (0.079) & 0.261 (0.080) & 1.253 (0.317) \\
\rowcolor{gray!15}
DW-RWMH  & 0.810 (0.045) & 0.398 (0.071) & 1.163 (0.156) & 0.815 (0.040) & 0.396 (0.066) & 1.155 (0.145) \\
\rowcolor{gray!15}
DW-EKCN  & \bf{0.247 (0.239)} & \bf{0.152 (0.190)} & \underline{1.054 (0.473)} & \bf{0.435 (0.119)} & \bf{0.207 (0.098)} & \underline{0.945 (0.279)} \\
\bottomrule
\end{tabular}%
}
\label{tab:t5-results}
\end{table}

\begin{figure}[h]
    \centering
    \includegraphics[width=1.0\textwidth]{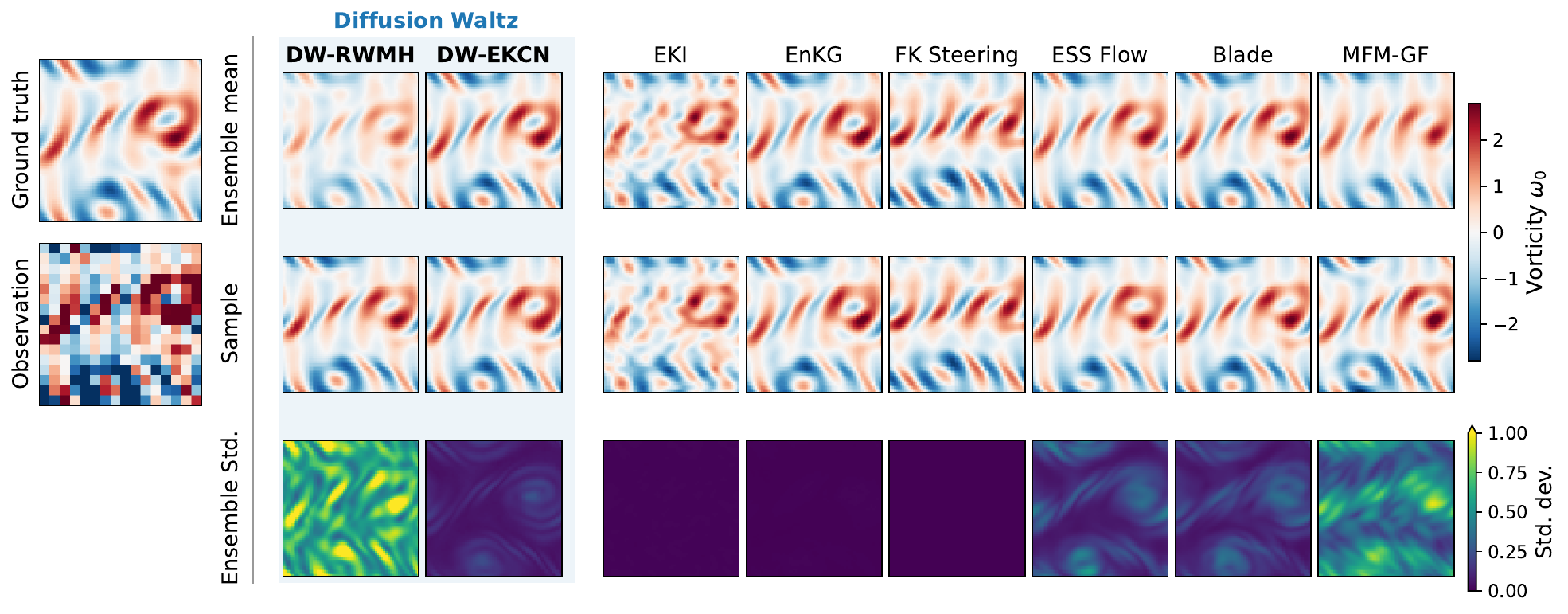}
    \caption{Comparison of diffusion waltz and baseline methods on the Navier-Stokes (NS) initial condition recovery problem in the low nonlinearity ($T=1$) / low noise ($\sigma_y / \sigma_0 = 0.1$) regime.}
    \label{fig:results-t1-s0.1}
\end{figure}

\begin{figure}[h]
    \centering
    \includegraphics[width=1.0\textwidth]{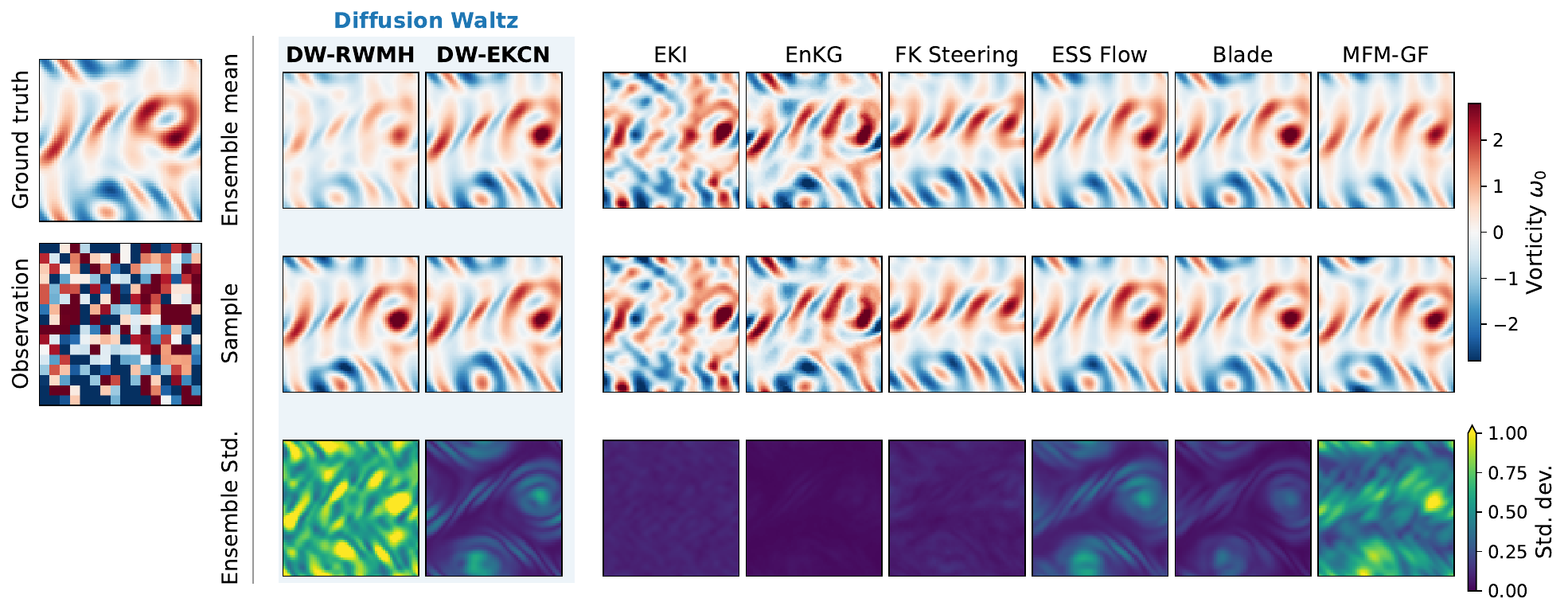}
    \caption{Comparison of diffusion waltz and baseline methods on the Navier-Stokes (NS) initial condition recovery problem in the low nonlinearity ($T=1$) / high noise ($\sigma_y / \sigma_0 = 1.0$) regime.}
    \label{fig:results-t1-s1.0}
\end{figure}

\begin{figure}[h]
    \centering
    \includegraphics[width=1.0\textwidth]{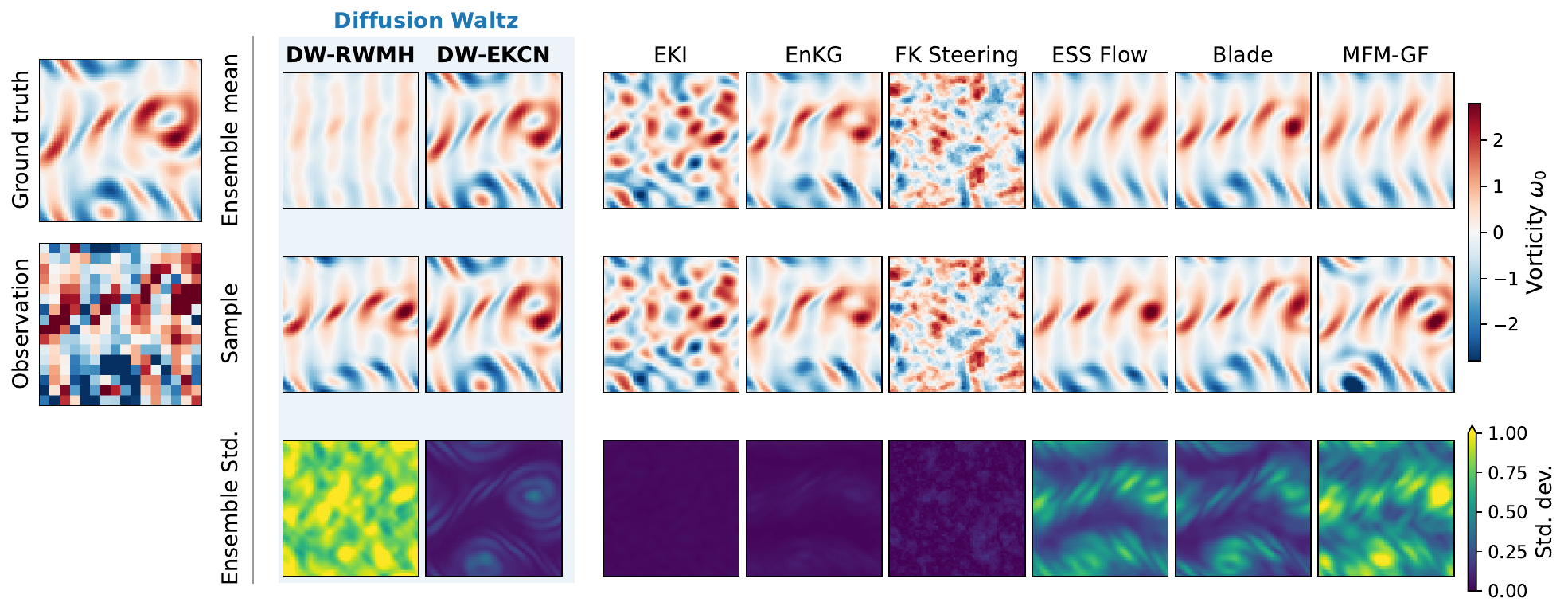}
    \caption{Comparison of diffusion waltz and baseline methods on the Navier-Stokes (NS) initial condition recovery problem in the high nonlinearity ($T=5$) / low noise ($\sigma_y / \sigma_0 = 0.1$) regime.}
    \label{fig:results-t5-s0.1}
\end{figure}


\begin{figure}[t]
    \centering
    \begin{subfigure}[b]{0.48\linewidth}
        \includegraphics[width=\linewidth]{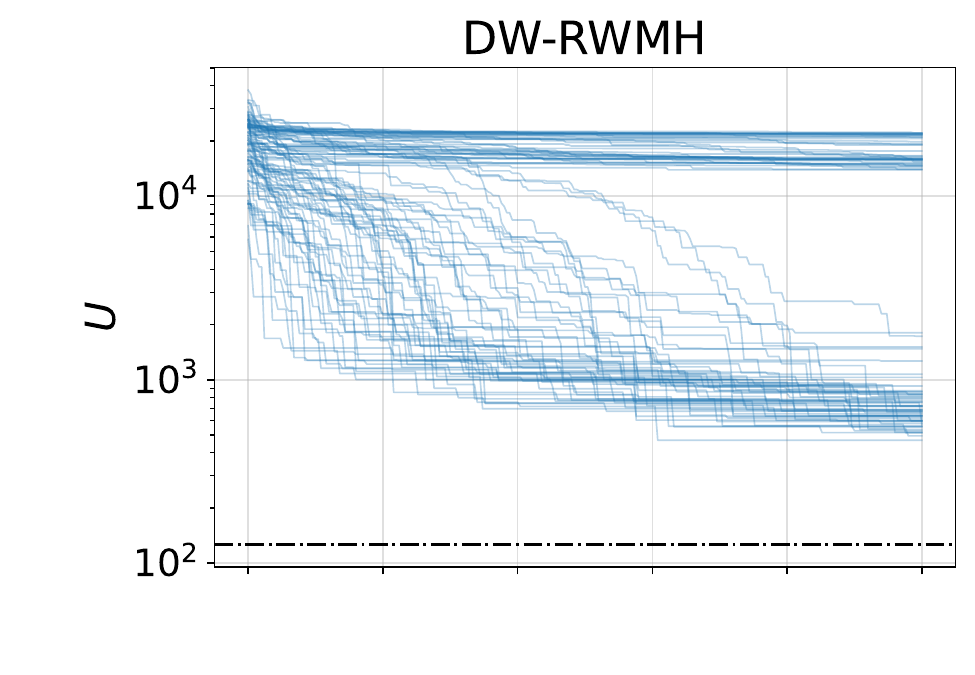}
        \caption{\quad $T=1$, \quad $\sigma_y / \sigma_0=0.1$}
    \end{subfigure}
    \hfill
    \begin{subfigure}[b]{0.48\linewidth}
        \includegraphics[width=\linewidth]{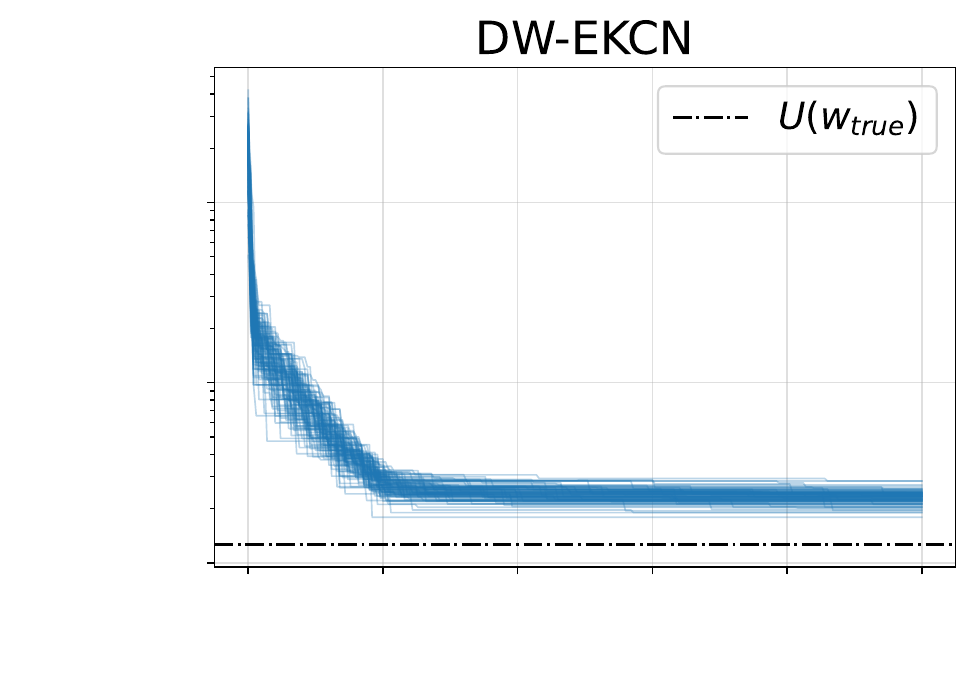}
        \caption{\quad $T=1$, \quad $\sigma_y / \sigma_0=0.1$}
    \end{subfigure}

    \vspace{0.5em}
    \begin{subfigure}[b]{0.48\linewidth}
        \includegraphics[width=\linewidth]{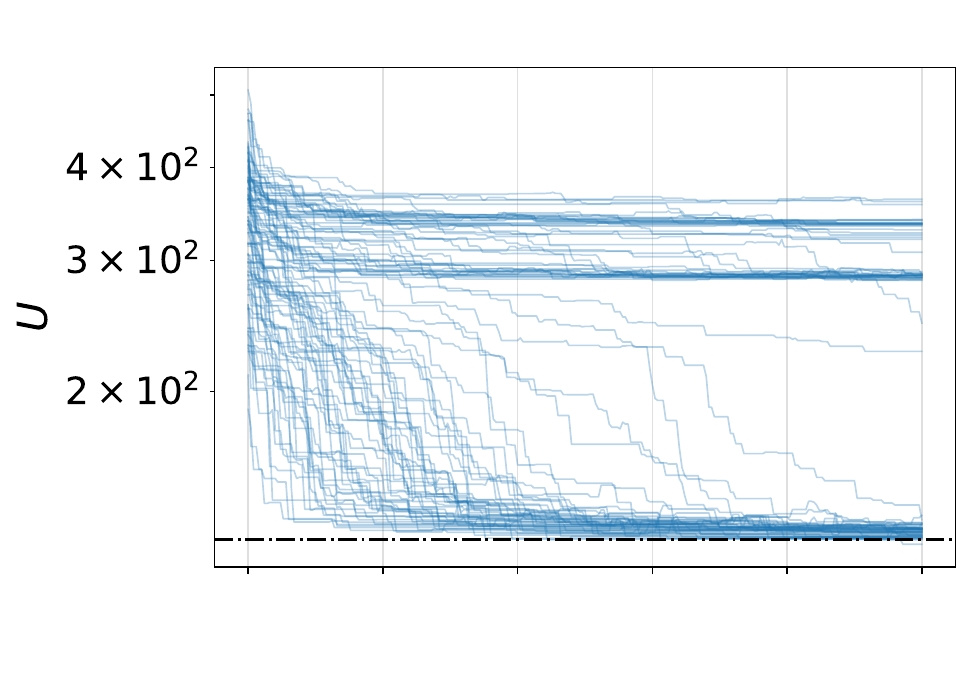}
        \caption{\quad $T=1$, \quad $\sigma_y / \sigma_0=1.0$}
    \end{subfigure}
    \hfill
    \begin{subfigure}[b]{0.48\linewidth}
        \includegraphics[width=\linewidth]{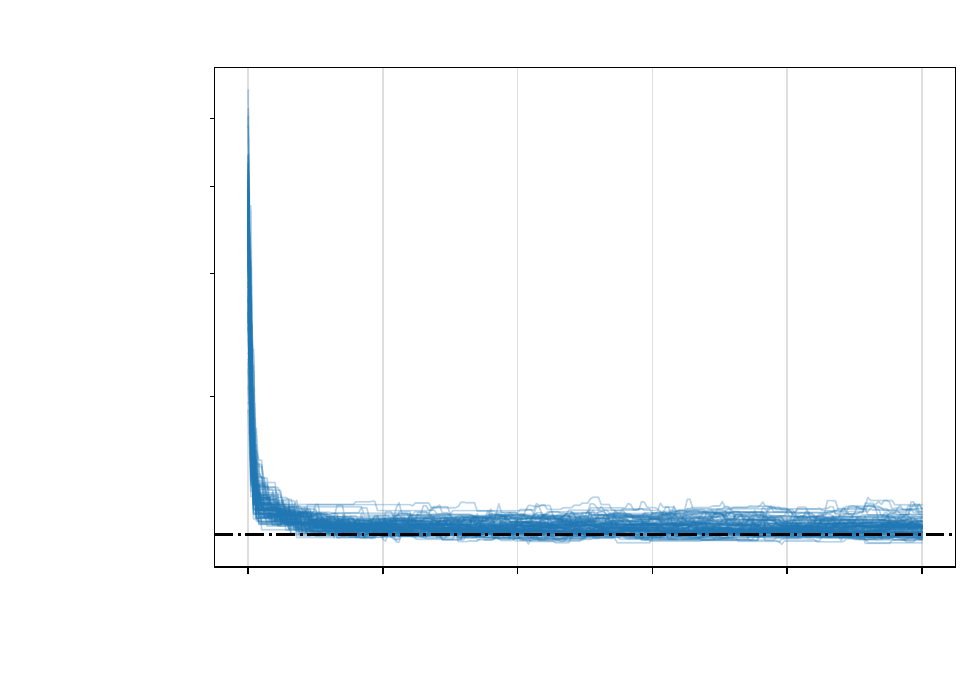}
        \caption{\quad $T=1$, \quad $\sigma_y / \sigma_0=1.0$}
    \end{subfigure}

    \vspace{0.5em}
    \begin{subfigure}[b]{0.48\linewidth}
        \includegraphics[width=\linewidth]{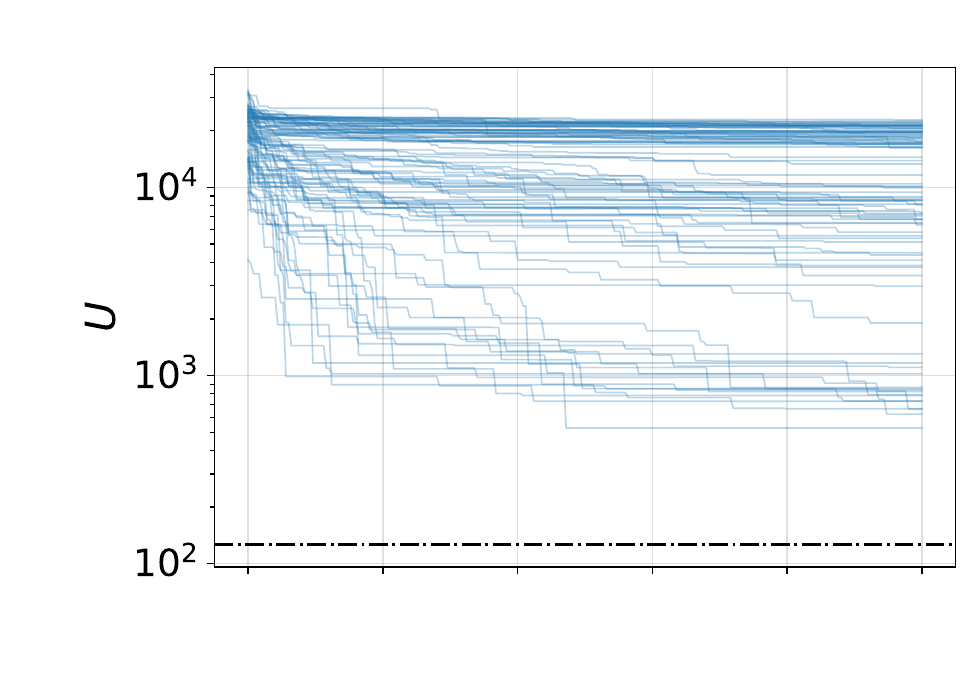}
        \caption{\quad $T=5$, \quad $\sigma_y / \sigma_0=0.1$}
    \end{subfigure}
    \hfill
    \begin{subfigure}[b]{0.48\linewidth}
        \includegraphics[width=\linewidth]{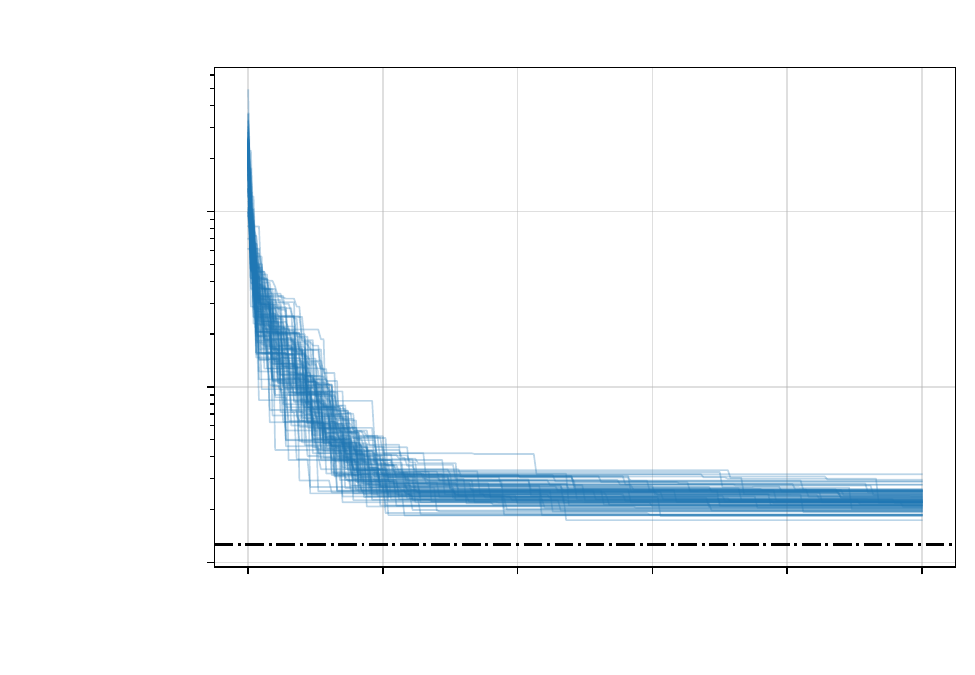}
        \caption{\quad $T=5$, \quad $\sigma_y / \sigma_0=0.1$}
    \end{subfigure}

    \vspace{0.5em}
    \begin{subfigure}[b]{0.48\linewidth}
        \includegraphics[width=\linewidth]{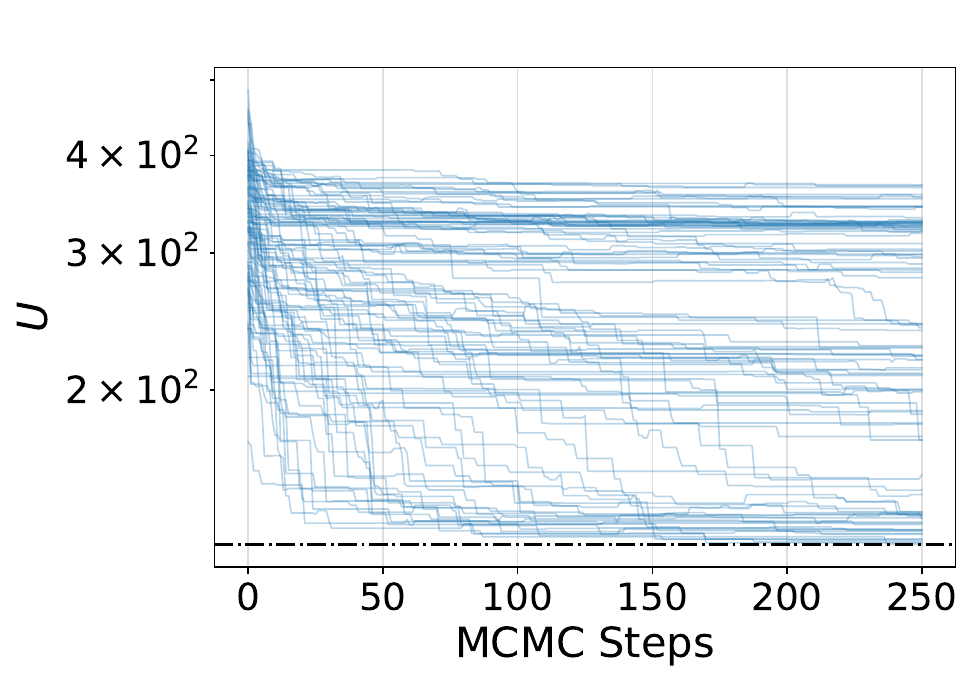}
        \caption{\quad $T=5$, \quad $\sigma_y / \sigma_0=1.0$}
    \end{subfigure}
    \hfill
    \begin{subfigure}[b]{0.48\linewidth}
        \includegraphics[width=\linewidth]{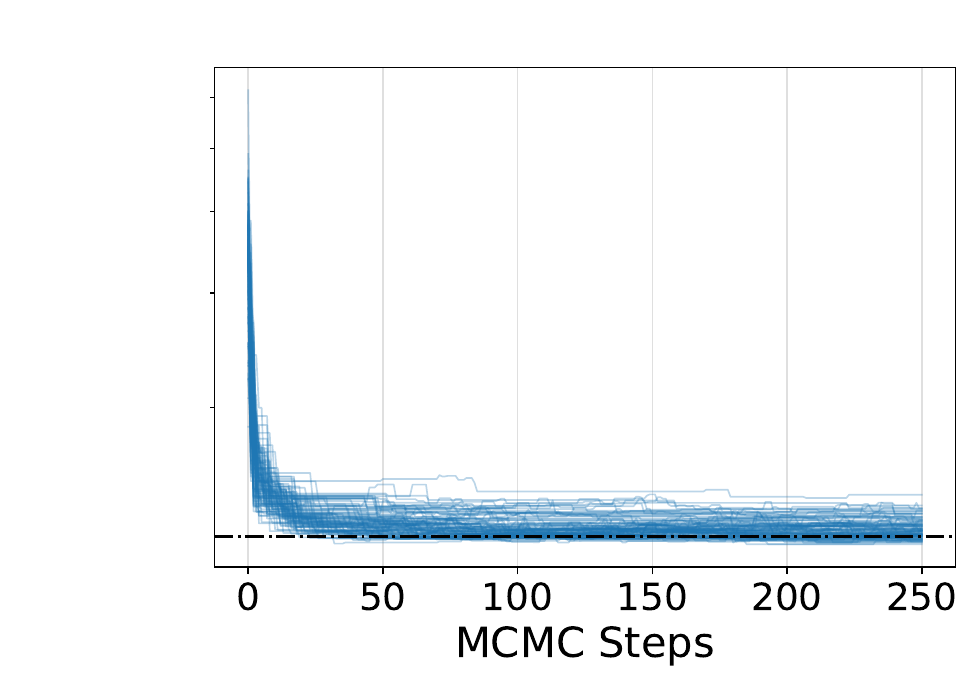}
        \caption{\quad $T=5$, \quad $\sigma_y / \sigma_0=1.0$}
    \end{subfigure}

    \caption{MCMC traces for the potential $U(x_0) := -\log p(y \mid x_0)$. We compare DW-RWMH (left) and DW-EKCN (right), across all four regimes. Many chains in DW-RWMH get stuck, unable to find a path to high likelihood/low potential regions, leading to slow convergence. Conditional noising in DW-EKCN alleviates this problem, leading to rapid convergence within the first 50 steps.}
    \label{fig:ns-mcmc-trace}
\end{figure}

\begin{figure}[t]
    \centering
    \begin{subfigure}[b]{0.48\linewidth}
        \includegraphics[width=\linewidth]{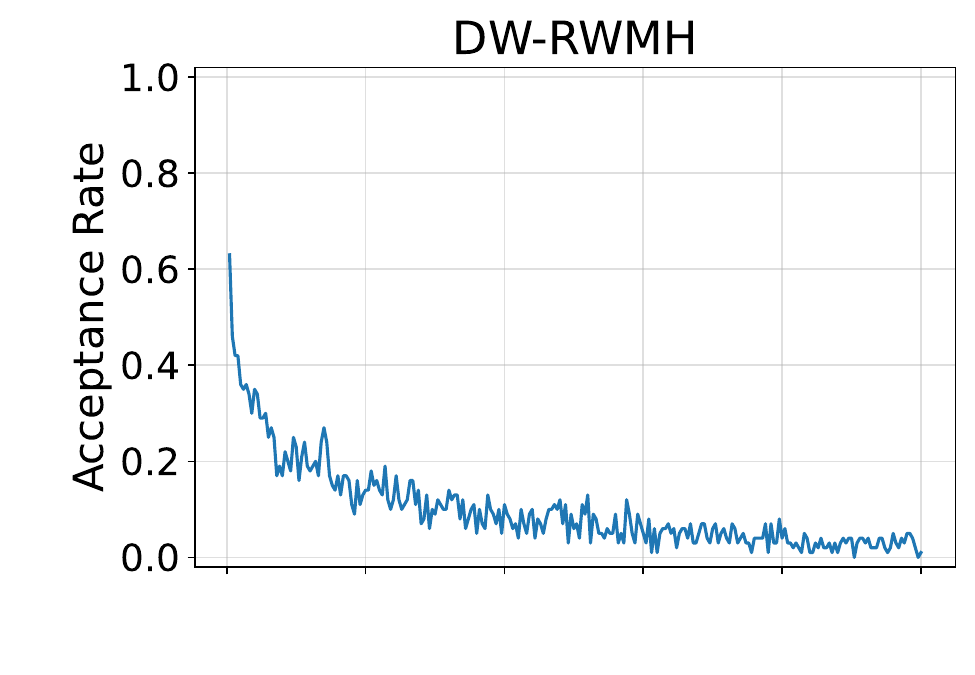}
        \caption{\quad $T=1$, \quad $\sigma_y / \sigma_0=0.1$}
    \end{subfigure}
    \hfill
    \begin{subfigure}[b]{0.48\linewidth}
        \includegraphics[width=\linewidth]{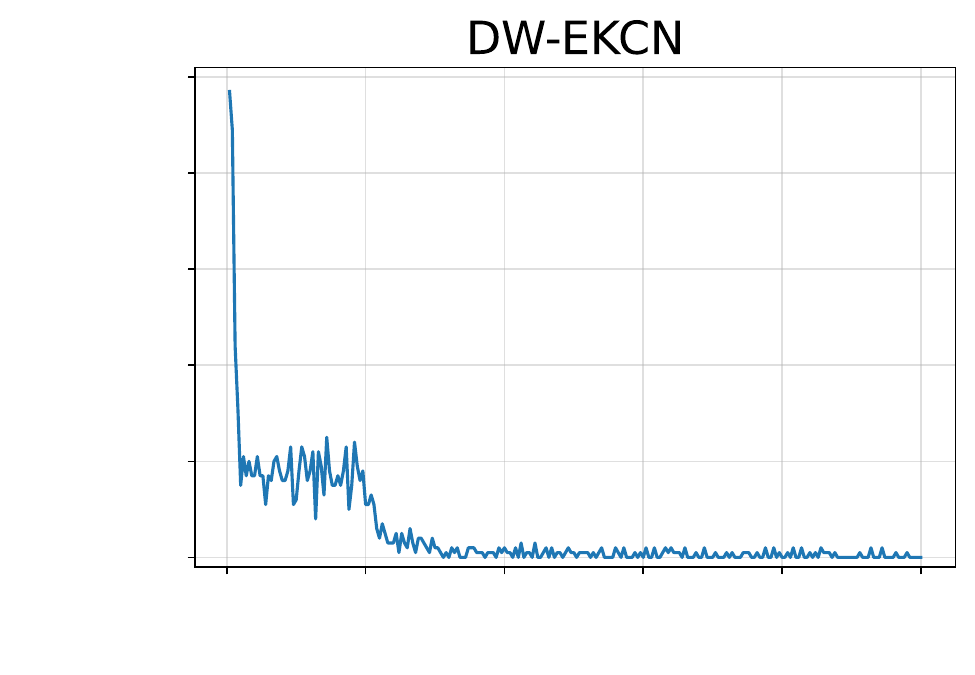}
        \caption{\quad $T=1$, \quad $\sigma_y / \sigma_0=0.1$}
    \end{subfigure}

    \vspace{0.5em}
    \begin{subfigure}[b]{0.48\linewidth}
        \includegraphics[width=\linewidth]{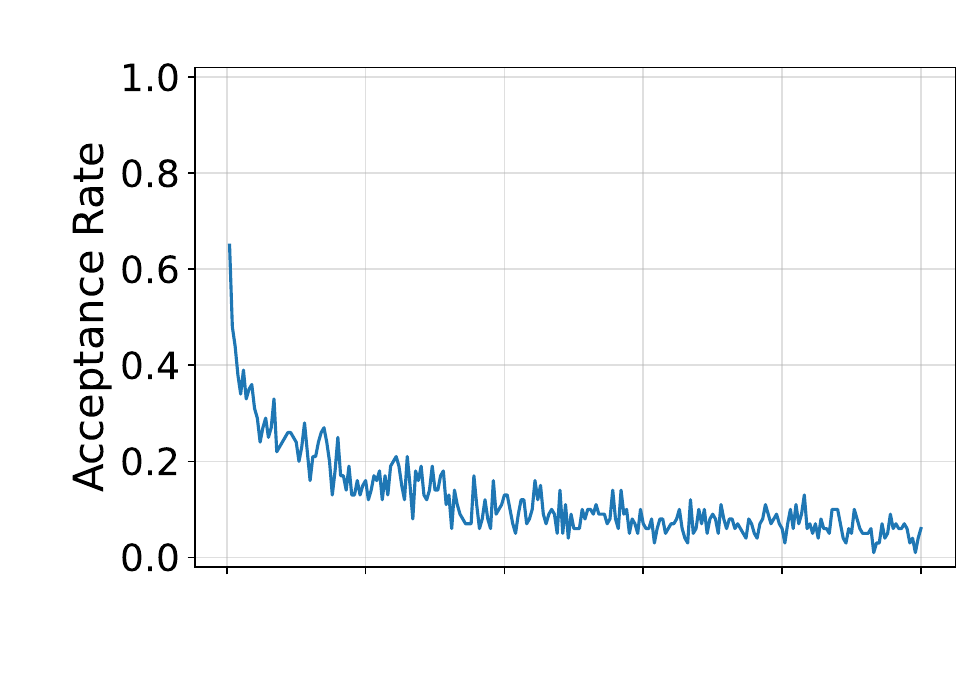}
        \caption{\quad $T=1$, \quad $\sigma_y / \sigma_0=1.0$}
    \end{subfigure}
    \hfill
    \begin{subfigure}[b]{0.48\linewidth}
        \includegraphics[width=\linewidth]{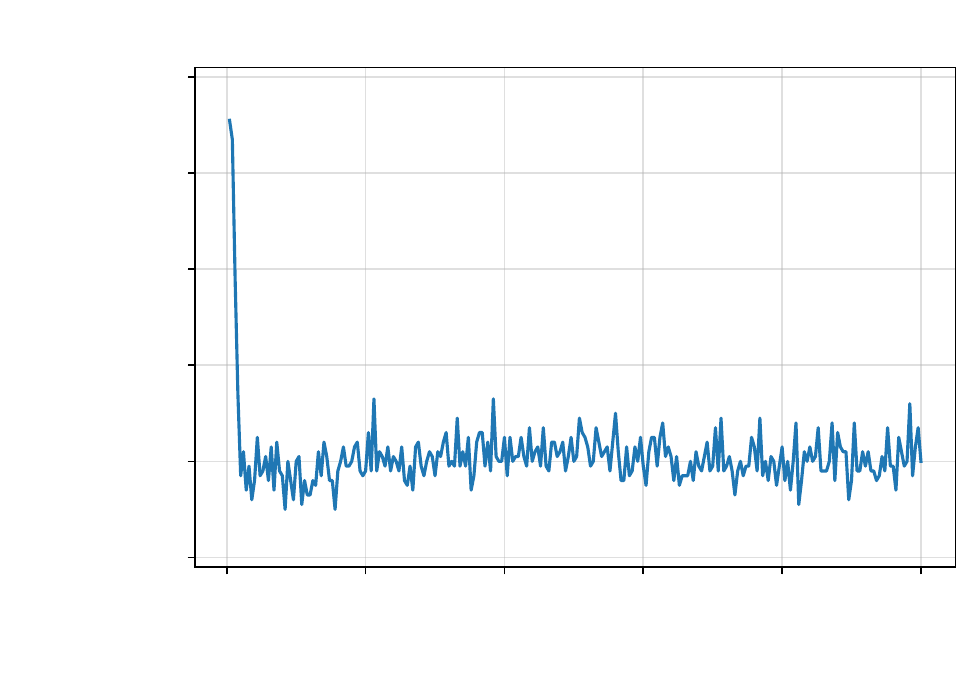}
        \caption{\quad $T=1$, \quad $\sigma_y / \sigma_0=1.0$}
    \end{subfigure}

    \vspace{0.5em}
    \begin{subfigure}[b]{0.48\linewidth}
        \includegraphics[width=\linewidth]{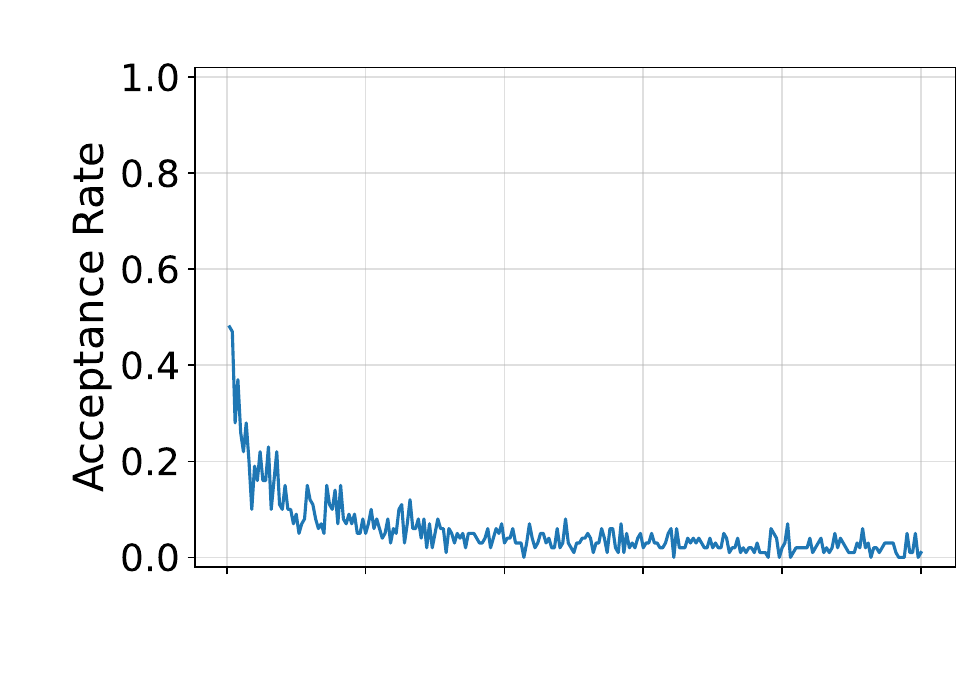}
        \caption{\quad $T=5$, \quad $\sigma_y / \sigma_0=0.1$}
    \end{subfigure}
    \hfill
    \begin{subfigure}[b]{0.48\linewidth}
        \includegraphics[width=\linewidth]{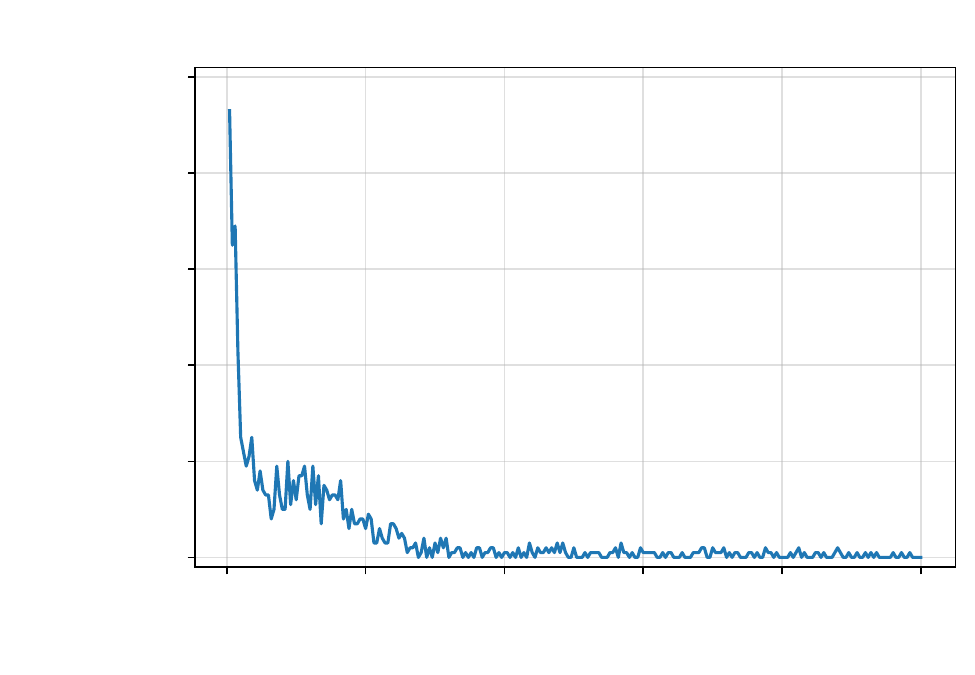}
        \caption{\quad $T=5$, \quad $\sigma_y / \sigma_0=0.1$}
    \end{subfigure}

    \vspace{0.5em}
    \begin{subfigure}[b]{0.48\linewidth}
        \includegraphics[width=\linewidth]{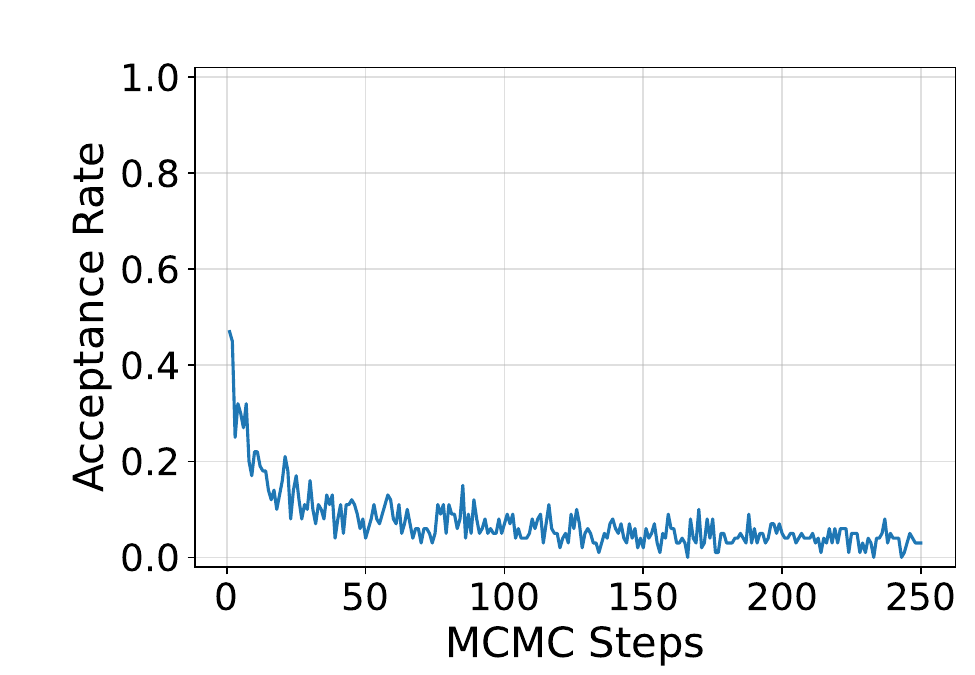}
        \caption{\quad $T=5$, \quad $\sigma_y / \sigma_0=1.0$}
    \end{subfigure}
    \hfill
    \begin{subfigure}[b]{0.48\linewidth}
        \includegraphics[width=\linewidth]{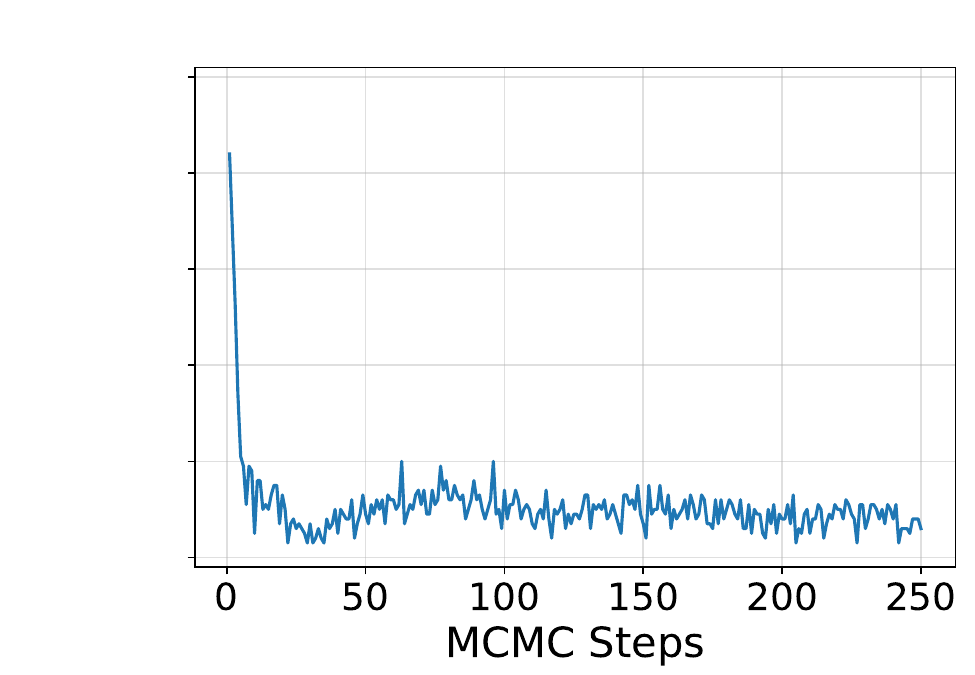}
        \caption{\quad $T=5$, \quad $\sigma_y / \sigma_0=1.0$}
    \end{subfigure}

    \caption{Instantaneous acceptance rates for DW-RWMH (left) and DW-EKCN (right), across all four regimes. DW-RWMH shows steady decay with MCMC steps. In the low noise regime, DW-EKCN has high acceptance rates initially, fluctuates around $20\%$ while the chains converge, then hits $\approx 0 \%$. In the high noise regime, acceptance rate is high initially, then immediately relaxes to $10 \sim 20\%$.}
    \label{fig:ns-acc-rate}
\end{figure}



\end{document}